\documentclass{article}

\usepackage{amsmath, amssymb, amsthm, amsfonts}

\usepackage{xspace}
\usepackage{boxedminipage}
\usepackage{epigraph}
\usepackage{framed}
\usepackage[framemethod=tikz]{mdframed}
\usepackage{titlesec}
\usepackage{lipsum}
\usepackage{algorithmic}
\usepackage{caption}
\usepackage{verbatim}
\usepackage{natbib}
\usepackage{comment}

\usepackage{microtype}
\usepackage{graphicx}
\usepackage{subfigure}
\usepackage{booktabs} 
\usepackage{hyperref}
\usepackage{cleveref}

\usepackage[accepted]{icml2021}

\icmltitlerunning{Spectral Analysis of the Neural Tangent Kernel for Deep Residual Networks}

\newcommand{\Real}{\mathbb{R}}
\newcommand{\Sphere}{\mathbb{S}}

\newcommand{\hh}{{\mathcal{H}}}

\newcommand{\bias}{\mathbf{b}}

\newcommand{\vv}{\mathbf{v}}

\newcommand{\x}{\mathbf{x}}

\newcommand{\z}{\mathbf{z}}

\newcommand{\norm}[1]{\left\lVert#1\right\rVert}

\newlength\myindent
\usepackage{enumerate}

\newtheorem{theorem}{Theorem}[section]
\newtheorem{lemma}[theorem]{Lemma}

\newtheorem*{theorem*}{Theorem}
\newtheorem*{claim*}{Claim}
\newtheorem{corollary}[theorem]{Corollary}

\newtheorem{definition*}{Definition}

\newcommand{\ntk}{\boldsymbol{k}}
\newcommand{\ntkl}[1]{{\boldsymbol{k}^{(#1)}}}
\newcommand{\tntkl}[1]{{\tilde{\boldsymbol{k}}^{(#1)}}}
\newcommand{\resntk}{\boldsymbol{r}}
\newcommand{\resntkl}[1]{{\boldsymbol{r}^{(#1)}}}
\newcommand{\lap}{\boldsymbol{k}_{Lap}}
\newcommand{\hlap}{\boldsymbol{k}_{HLap}}
\newcommand{\Rd}{\Real^d}
\newcommand{\Spdm}{\Sphere^{d-1}}
\newcommand{\cpie}{\frac{\sqrt{2}}{\pi}}

\newcounter{casecount}
\newenvironment{proof-sketch}{\noindent{\it Proof Sketch.}\hspace*{1em}}{\qed\bigskip}

\usepackage{etoolbox}
\AtBeginEnvironment{proof}{\setcounter{casecount}{0}}

\begin{document}
\twocolumn[
\icmltitle{Spectral Analysis of the Neural Tangent Kernel for Deep Residual Networks}




\begin{icmlauthorlist}
\icmlauthor{Yuval Belfer}{}
\icmlauthor{Amnon Geifman}{}
\icmlauthor{Meirav Galun}{}
\icmlauthor{Ronen Basri}{}
\\ \vskip 0.1cm
\icmlauthor{{\small Weizmann Institute of Science}}{}
\\
\icmlauthor{{\small Rehovot, Israel}}{}
\end{icmlauthorlist}


\icmlcorrespondingauthor{Yuval Belfer}{ep@eden.co.uk}

\icmlkeywords{Machine Learning, ICML}

\vskip 0.3in
]
\begin{abstract}
    Deep residual network architectures have been shown to achieve superior accuracy over classical feed-forward networks, yet their success is still not fully understood. Focusing on massively over-parameterized, fully connected residual networks with ReLU activation through their respective neural tangent kernels (ResNTK), we provide here a spectral analysis of these kernels. Specifically, we show that, much like NTK for fully connected networks (FC-NTK), for input distributed uniformly on the hypersphere $\Spdm$, the eigenfunctions of ResNTK are the spherical harmonics and the eigenvalues decay polynomially with frequency $k$ as $k^{-d}$. These in turn imply that the set of functions in their Reproducing Kernel Hilbert Space are identical to those of FC-NTK, and consequently also to those of the Laplace kernel. We further show, by drawing on the analogy to the Laplace kernel, that depending on the choice of a hyper-parameter that balances between the skip and residual connections ResNTK can either become spiky with depth, as with FC-NTK, or maintain a stable shape.
\end{abstract}

\section{Introduction}

Deep residual networks (ResNets), first introduced in~\cite{he2016deep}, are to date amongst the most effective network architectures for image understanding~\cite{howard2019searching,radosavovic2020designing,tan2019mnasnet} as well as for other tasks~\cite{greenfeld2019amg,Siravenha2019resmlp}. These networks use blocks of two or three layers with skip connections such that the input to each block is added to its output (called the \textit{residual}) and the sum is passed to the next block. These architectural changes allowed researchers to train networks with hundreds, and even thousands of layers and to achieve unprecedentedly accurate classification results on the competitive ImageNet dataset \cite{he2016deep,he2016identity}.

The reasons for the advantage of residual  over classical feed-forward architectures are not yet fully understood. Several papers argue that skip connections alleviate the problem of \textit{vanishing gradients}, which is prevalent in classical deep architectures \cite{balduzzi2017shattered,veit2016residual}. 
Subsequent work showed that ResNets can avoid spurious local minima \cite{liu2019towards}, while \cite{li2018visualizing} showed, by empirically visualizing the loss landscape, that skip connections make the loss smoother.

In this work we examine residual networks from the perspective of the neural tangent kernels. As with many existing network models, residual network applications are typically over-parameterized. \cite{he2016deep}'s implementation, for example, trains a network with roughly 60M trainable parameters on the 1.2M images of ImageNet. Recent work~\cite{jacot2018} suggested that massively overparameterized neural networks behave similarly to kernel regressors with a family of kernels called \textit{Neural Tangent Kernels} (NTKs). \cite{Huang2020WhyDD,tirer2020kernelbased} proved that fully connected residual networks of infinite width converge to such kernel, which we here call \textit{ResNTK}, and provided a closed form derivation.

Kernel regression is characterized by the set of functions in the corresponding Reproducing Kernel Hilbert Space (RKHS) and by the norm induced in this space. These in turn are determined by the eigenfunctions and eigenvalues of the respective kernel under the uniform measure, with the decay rate of the eigenvalues playing a particularly important role. In this paper we prove that the eigenfunctions of ResNTK on the hypersphere $\Spdm$ are the spherical harmonics and that with ReLU activations the eigenvalues decay polynomially with frequency $k$ at the rate of $k^{-d}$, thus characterizing the set of functions in the corresponding RKHS. We conclude that this set of functions is identical to the functions in the RKHS of NTK of classical, fully connected networks (denoted \textit{FC-NTK}) \cite{basri2020frequency,bietti2020deep}, and, as is implied by previous work \cite{geifman2020similarity,bietti2020deep,chen2020laplace}, also to those of the Laplace kernel, restricted to $\Spdm$. We further discuss how this characterization extends outside of the hypersphere to $\Rd$.

Various properties of ResNTK appear to critically depend on a choice of hyperparameter $\alpha$, which balances between the residual and skip connections. In particular, we examine these properties when $\alpha$ is either constant or decaying with the depth of the corresponding network and make the following additional contributions:
\begin{enumerate}
    \item With no bias and a decaying $\alpha$ ($\alpha=L^{-\gamma}$ and $0.5<\gamma\le 1$ where $L$ denotes the number of hidden layers in the corresponding network), deep ResNTK is significantly biased toward the even frequencies. Specifically, with deep ResNTK the leading eigenfunctions beyond frequencies 0,1, and 2 are the even frequencies, and eigenfunctions of odd frequency have significantly lower eigenvalues. Ultimately when the depth $L \rightarrow \infty$ ResNTK converges to a two-layer FC-NTK, for which with no bias all the eigenvalues corresponding to odd frequency eigenfunction (except frequency 1) vanish. Such a parity difference is not observed if bias is used, if $\alpha=1/\sqrt{L}$, or if $\alpha$ is constant.
    \item Through the analogy to the Laplace kernel we can show the condition for which ResNTK become spiky. Specifically, we show that, with a decaying $\alpha=L^{-\gamma}$ with $0.5 \le \gamma \le 1$ ResNTK maintains a roughly stable shape, but becomes spiky with deep architectures if $\alpha$ is constant independent of depth. With this choice ResNTK exhibits the same behavior as FC-NTK. Our expreiments indeed indicate that with real datasets (UCI, CIFAR-10 and SVHN) a spiky kernel achieves inferior classification results compared to less steep kernels, implying that with FC-NTK and ResNTK with a constant $\alpha$ deep architectures are in fact inferior to shallow ones. 
\end{enumerate}


\section{Previous work}


Existing neural network models are typically applied with many more learnable parameters than training data items, yet somewhat counter-intuitively they successfully generalize to unseen data. Attempting to explain this phenomenon \cite{jacot2018} showed that infinite width networks whose parameters do not change much from their initial values behave like kernel regression with novel kernels called the Neural Tangent Kernels. Specifically, for an input $\x \in \Rd$ and learnable parameters $\theta \in \Real^m$, denote the network by $f(\x,\theta)$, then the corresponding NTK is given by
\begin{equation*}
    \mathbb{E}_{\theta \sim \mathcal{P}} \left< \frac{\partial f(\x_i,\theta)}{\partial \theta}, \frac{\partial f(\x_j,\theta)}{\partial \theta} \right>,
\end{equation*}
where $\x_i$ and $\x_j$ is a training pair, and the expectation is over the distribution $\mathcal{P}$ with which $\theta$ is initialized (typically the standard normal distribution). We note that the relevance of these models, referred to as \textit{lazy training}, to realistic neural networks is the subject of an ongoing debate (see, e.g., \cite{chizat2019lazy,lee2020finite}). 

Subsequent work showed that very wide networks of finite width converge to a global minimum~\cite{du2019gradient,allen2019convergence,chizat2019lazy} and further characterized the speed of convergence as a function of the data distribution and the frequency of the target function \cite{arora2019finegrained,Basri2019TheCR,basri2020frequency}. In particular, for data distributed uniformly in the hypersphere $\Spdm$, it was shown that the eigenfunctions of FC-NTK are the spherical harmonics and the eigenvalues decay at the rate of $k^{-d}$, where $k$ denotes frequency \cite{bietti2019inductive,bietti2020deep}. This completely characterizes the set of functions in the RKHS of FC-NTK. Subsequent work showed that this set of functions is identical to the functions in the RKHS of the classical Laplace kernel \cite{geifman2020similarity,bietti2020deep,chen2020laplace}. Our paper extends these results to NTK of residual networks  of any depth.

Several recent studies examined the behavior of over-parameterized residual networks. \cite{du2019gradient,zhang2019stability} showed that very wide ResNets of finite size converge to their global minima.
\cite{Huang2020WhyDD,tirer2020kernelbased} derived a formula for ResNTK. \cite{tirer2020kernelbased}'s analysis further suggested that ResNTK gives rise to a class of smoother function than FC-NTK. \cite{Huang2020WhyDD} showed that FC-NTK becomes spiky for deep networks, indicating that learning with these kernels becomes degenerate, while ResNTK remains stable with depth. Our work shows that the functions in the RKHS of both ResNTK and FC-NTK have the same smoothness properties. Moreover, we show that the specific choice of $\alpha$, the hyper-parameter that balances between the skip and residual connections, has a significant effect on the shape of ResNTK for deep architecture, so, for example, with constant $\alpha$ ResNTK too becomes spiky with depth.

Understanding the spectrum of a kernel is useful for a number of objectives. It indicates whether a kernel exhibits a frequency bias~\cite{cao2019,rahaman2019spectral,Xu2019}, it provides an estimate of the number of gradient descent iterations needed to learn certain target functions~\cite{Basri2019TheCR}, and it can be used to estimate the generalization error obtained by using the kernel as a minimum interpolant regressor (ridge-less kernel regression). For example, \cite{liang2020just,liang2019risk,pagliana2020interpolation} analyzed the bias-variance interplay of minimum norm interpolation with a growing number of samples when the dimension is either fixed or growing at the same rate.

\comment{
\begin{enumerate}
    \item Original ResNet Paper \url{https://openaccess.thecvf.com/content_cvpr_2016/papers/He_Deep_Residual_Learning_CVPR_2016_paper.pdf}  
    \item Raja Gyries's paper \url{https://arxiv.org/pdf/2009.10008.pdf}
    \item Du optimization of ResNet  \url{https://arxiv.org/pdf/1909.04653.pdf},\url{https://arxiv.org/pdf/1811.03804.pdf}
    \item ResNTK - \url{https://arxiv.org/pdf/2002.06262.pdf}
    \item Allen-Zhu Resnet \url{https://papers.nips.cc/paper/2019/file/5857d68cd9280bc98d079fa912fd6740-Paper.pdf}
    \item Zhang's optimization of ResNet- (also overparametrized) \url{https://openreview.net/pdf?id=HJe-oRVtPB}. Another similar paper by Zhang \url{https://arxiv.org/pdf/1903.07120.pdf}
    \item Mean field analysis of ResNet \url{https://arxiv.org/pdf/2003.05508.pdf}
\end{enumerate}
How is our work related to previous results?
\yb{Works from NIPS that are all experimental: \url{https://arxiv.org/pdf/2010.15110.pdf}, \url{https://proceedings.neurips.cc/paper/2020/file/ad086f59924fffe0773f8d0ca22ea712-Paper.pdf}}

}

\section{Preliminaries}
We consider positive definite kernels $\ntk: \Rd \times \Rd \rightarrow \Real$ over inputs $\x, \z \in \Rd$. $\ntk$ is called zonal if when $\x,\z$ are restricted to the hypersphere $\Spdm$ $\ntk$ can be expressed as a function of $\x^T\z$. In such case we overload our definition of $\ntk$ defining also $\ntk: [-1,1] \rightarrow \Real$ by letting $u=\x^T\z$ and writing $\ntk(\x,\z)=\ntk(u)$. To avoid unnecessary scalings, a good practice is to normalize the kernel such that $\ntk(1) = 1$. The eigenfunctions and eigenvalues derived in this paper are with respect to the uniform measure on the hypersphere $\Spdm$, or with respect to radial distributions in $\Rd$. Note however that the resulting RKHS definition is independent of data distribution. The kernels we use in this paper are ResNTK and FC-NTK, denoted respectively by $\resntk$ and $\ntk$, as well as the Laplace kernel (denoted $\lap$), with superscripts denoting the number of hidden layers, e.g. $\ntk^{(L)}$, i.e., $L=1$ corresponds to a network with one hidden layer (i.e., a two-layer network). Except when noted our kernels will correspond to networks with no bias. All proofs are the deferred to the supplementary material.

\subsection{NTK for FC Networks}

A fully-connected neural network (also called multilayer perceptron, MLP) with $L$ hidden layers and $m$ units in each hidden layer is expressed as
\begin{align*}
    f(\theta,\x) &= \vv^T \x_{L}\\
    \x_\ell & = \sqrt{\frac{c_\sigma}{m}}
    \sigma \left( W^{(l)}\x_{\ell-1} \right), ~~ \ell \in [L]\\
    \x_0&=\x.
\end{align*}
The network parameters $\theta$ include $W^{(1)},W^{(2)},...,W^{(L)}$, where $W^{(1)} \in \Real^{d \times m}$, $W^{(\ell)} \in \Real^{m \times m}$ ($2 \le \ell \le L$), and $\vv \in \Real^{m}$.  We denote by $\sigma$ the ReLU activation function and by $c_{\sigma} = 1/\left( \mathbb{E}_{z \sim \mathcal{N}(0,1)} [\sigma(z)^2] \right)=2$. The  network parameters are initialized randomly with ${\cal N}(0,I)$.

\cite{jacot2018} showed that when the width $m \rightarrow \infty$ the network behaves like kernel regression with the neural tangent kernel. \cite{bietti2019inductive} showed that this kernel, denoted for $\x,\z \in \Rd$ by $\ntkl{L}(\x,\z)$, is homogeneous of degree 1 and zonal, so that $\ntkl{L}(\x,\z) = \|\x\|\z\|\ntkl{L}(u)$, where $u =
\frac{\x^T\z}{\|\x\|\|\z\|} \in [-1,1]$. The (normalized) kernel is defined by
\begin{equation*}
    \ntkl{L}(u) = \frac{1}{L+1} \tntkl{L}(u)
\end{equation*}
with the recursive formula
\begin{align}
    \tntkl{\ell}(u) &= \tntkl{\ell-1}(u)\kappa_0(\Sigma^{(\ell-1)}(u))+\Sigma^{(\ell)}(u)\\
    \Sigma^{(\ell)}(u) &= \kappa_1(\Sigma^{(\ell-1)}(u)), ~~ \ell \in [L]. \nonumber
\end{align}
The functions $\kappa_1,\kappa_0$ are the arc-cosine kernels \cite{Cho2009NIPS}, defined as
\begin{align}
    \kappa_0(u) &= \frac{1}{\pi}(\pi-acos(u))
    \label{eq:Kappa0}\\ 
    \kappa_1(u) &= \frac{1}{\pi}\left(u \cdot (\pi-acos(u)) + \sqrt{1-u^2}\right), 
    \label{eq:Kappa1}
\end{align}
and $\tntkl{0}(u)=\Sigma^{(0)}(u)=u$.

\subsection{NTK for residual networks}

For the definition of a fully connected residual network we follow the formulation of \cite{Huang2020WhyDD,tirer2020kernelbased}. Below we include bias, but except when noted we will work with a bias-free formulation (i.e., $\tau=0$).
\begin{eqnarray*}
g(\x,\theta) & =& \vv^T \x_L\\
\x_\ell & =& \x_{\ell-1} + \alpha \sqrt{\frac{1}{m}} V_\ell \, \sigma \left(\sqrt{\frac{2}{m}} W_\ell \x_{\ell-1}  + \tau \bias_\ell \right)\\
\x_0 & =& \sqrt{\frac{1}{m}}A \x,
\end{eqnarray*}
for $\ell \in [L]$ with parameters $A \in \Real^{m \times d}$, $V_\ell, W_\ell \in \Real^{m \times m}$ and $\vv \in \Real^m$, and $\sigma(\cdot)$ is the ReLU function.
$\alpha$ is a constant hyper-parameter.
\cite{Huang2020WhyDD,du2019gradient} suggested to set this constant according to $\alpha=L^{-\gamma}$ with $0.5 \leq \gamma \leq 1$. In contrast, \cite{he2016deep}'s implementation uses $\alpha=1$ (and an additional ReLU function applied to $V_\ell \, \sigma(.)$). Recent work argued that setting $\alpha$ to decay with depth is enforced in practice through suitable small initialization of the residual parameters or by applying normalization blocks~\cite{zhang2019residual}.

Adopting \cite{Huang2020WhyDD}'s derivation, we assume that both $A$ and $\vv$ are fixed at their initial values and that $V_\ell$, $W_\ell$ and $\bias$ are learned, with all parameters initialized with the standard normal distribution except for the bias terms $\bias_\ell$, which are initialized at 0. Let $\x,\z \in \Rd$. The respective NTK, denoted $\resntkl{L}(\x,\z)$, is given by 
\begin{align}
    \resntkl{L}(\x,\z) &= C \sum_{\ell=1}^L B_{\ell+1}(\x, \z) \left[v_{\ell-1}(\x,\z) \nonumber \kappa_1(u_{\ell-1}(\x,\z)) \right.\\ & + \left. (K_{\ell-1}(\x,\z) + \tau^2) \kappa_0(u_{\ell-1}(\x,\z))\right],
\label{eq:ResNTK}
\end{align}
where for $\ell \in [L]$ we let \begin{eqnarray*}
    v_{\ell}(\x,\z) &=& \sqrt{K_{\ell}(\x,\x)K_{\ell}(\z,\z)} \\ u_{\ell}(\x,\z) &=& \frac{K_{\ell}(\x,\z)}{v_{\ell}(\x,\z)}\\
    K_{\ell}(\x,\z) &=& K_{\ell-1}(\x,\z) + \alpha^2v_{\ell-1}(\x,\z)\kappa_1(u_{\ell-1})\\
    B_{\ell}(\x,\z) &=& B_{\ell+1}(\x,\z)[1+\alpha^2\kappa_0(u_{\ell-1})]\\
    K_{0}(\x,\z) &=& \x^T\z\\
    B_{L+1}(\x,\z) &=& 1\\
    C &=& \frac{1}{2L(1+\alpha^2)^{L-1}},
\end{eqnarray*}
and $\kappa_0$ and $\kappa_1$ are defined in \eqref{eq:Kappa0}-\eqref{eq:Kappa1}.

We note that with this model with $L=1$ ResNTK is equal to FC-NTK, i.e., $\resntkl{1}=\ntkl{1}$.


\comment{
\subsection{Spectral analysis of FC-NTK}

\rb{Repeat Thm 1 and Cor. 3 from Bietty and Bach (do we need to fix it?). 
In the appendix include eqs 12-13 along with the analogous expressions at -1.} \yb{The decay of the eigenvalues is still true. We need to fix their claim in the conclusions: "which remains the same for kernels arising from fully-connected ReLU networks of varying depths. This
result suggests that the kernel approach is unsatisfactory for understanding the power of depth in fully connected
networks." Actually, when I think about it, even the name of their paper, "Deep equals Shallow", is only true for the eigenvalues decay...}

Since FC-NTK is zonal, its eigenfunctions on the hypersphere $\Spdm$ are the spherical harmonics~\cite{xie2017diverse}. Recent work by \cite{bietti2020deep,chen2020laplace,geifman2020similarity} proved that the corresponding eigenvalues decay at the rate of $k^{-d}$, where $k$ denotes frequency. In particular, 

\cite{bietti2020deep} used Theorem \ref{thm: BBdeacy} to characterize the eigenvalue decay of FC-NTK:
\begin{theorem}[\cite{bietti2020deep}] \label{thm: BBNTK}
    For $L \ge 2$, the FC-NTK satisfies \eqref{eq:BBaround1} and \eqref{eq:BBaroundm1} with $|c_1| \ne |c_{-1}|$ and $\nu = 0.5$.
\end{theorem}
\rb{This is true with $L \ge 2$ but not with $L=1$, right?} \yb{Right} \rb{What are the coefficients for $L=1$?}
\begin{corollary}[\cite{bietti2020deep}]
The eigenvalues of FC-NTK decay at the rate of $k^{-d}$.
\end{corollary}
}

\section{Spectral Analysis of ResNTK}  \label{sec:spectral}

In this section we characterize the RKHS of ResNTK. In particular, we prove that the eigenfunctions of ResNTK are (scaled) spherical harmonics and that its eigenvalues decay with frequency $k$ at the rate of $k^{-d}$.

\subsection{Eigenfunctions of ResNTK}

\begin{theorem} \label{thm:ResNTKeigenfunctions}
    Bias-free ResNTK is homogeneous of degree 1 and zonal, i.e., $\resntk(\x,\z)=\|\x\|\|\z\|\resntk\left(\frac{\x^T\z}{\|\x\|\|\z\|}\right)$. Its eigenfunctions under the uniform measure in $\Spdm$ are the spherical harmonics.
\end{theorem}
The proof of this theorem, given in the supplementary material, relies on propagating these properties through the recursive definition of ResNTK. Finally, the spherical harmonics are eigenfunctions for any zonal kernel (see, e.g., \cite{gallier2009notes}). 

The following Theorem extends the eigen-decomposition of ResNTK to $\Rd$. 
\begin{theorem}
\label{thm:eig_outofsphere}
Let $p(r)$ be a decaying density on $[0,\infty)$ such that 
$0 < \int_0^\infty p(r)r^2 dr<\infty$ and $\x,\z\in\Rd$.
Then the eigenfunctions of the bias-free ResNTK $\resntk(\x,\z)$ with respect to $p(\|\x\|)$ are given by $\Psi_{k,j}=a\|\x\|Y_{k,j}\left(\frac{\x}{\|\x\|}\right)$ where $Y_{k,j}$ are the spherical harmonics in $\Spdm$ and the normalizing constant $a\in\Real$ depends on $p(r)$.
\end{theorem}
The proof of this theorem relies on the homogeneity of ResNTK and is immediate from \cite{geifman2020similarity}(Theorem 5 therein).

The consequence of Theorems~\ref{thm:ResNTKeigenfunctions} and~\ref{thm:eig_outofsphere} is that the bias-free ResNTK admits the following Mercer decomposition:
\begin{equation*}
    \resntk(\x,\z) = a^2 \sum_{k=0}^\infty \lambda_k \sum_{j=1}^{N(d,k)} \|\x\|Y_{kj}\left(\frac{\x}{\|x\|}\right) \|\z\|Y_{kj}\left(\frac{\z}{\|z\|}\right),
\end{equation*}
where $N(d,k)$ denotes the number of spherical harmonics of frequency $k$ in $\Spdm$. Note that this decomposition also ensures that the eigenvalues for the bias-free ResNTK in $\Rd$ are identical to those on $\Spdm$.

\subsection{Eigenvalue decay for ResNTK}

We next turn to characterizing the asymptotic behavior of the eigenvalues of ResNTK. This is our main theorem, and it is given below.

\begin{theorem} \label{thm:ResNTKdeacy}
The eigenvalues $\lambda_k$ of ResNTK, $\resntk(\x,\z)$, decay at the rate of $k^{-d}$ where $k$ denotes frequency.
\end{theorem}

The proof of this theorem uses a theorem proved recently by \cite{bietti2020deep}, which for certain zonal kernels relates the decay rate of the eigenvalues of a kernel to its infinitesimal tendency near $\pm 1$. \cite{bietti2020deep} used this theorem to derive the eigenvalue decay of FC-NTK for deep networks. Below we review the theorem and provide additional lemmas, which together allow us to prove Theorem~\ref{thm:ResNTKdeacy}.

\begin{theorem}[\cite{bietti2020deep}] \label{thm: BBdeacy}
    Let $\kappa : [-1,1] \xrightarrow{} \mathbb{R}$ be a $C^{\infty}$ function on $(-1,1)$ that has the following asymptotic expansions around $\pm1$
    \begin{eqnarray}
        \label{eq:BBaround1}
        \kappa(1-t) &=& p_1(t)+c_1t^{\nu}+o(t^{\nu})\\
        \label{eq:BBaroundm1}
        \kappa(-1+t) &=& p_{-1}(t)+c_{-1}t^{\nu}+o(t^{\nu})
    \end{eqnarray}
    for $t \ge 0$, where $p_1, p_{-1}$ are polynomials and $\nu > 0$ is not an integer. Let $\mu_k$ denote an eigenvalue of $\kappa$ corresponding to a spherical harmonic eigenfunction of frequency $k$. Then, there is an absolute constant $C(d, \nu)$ depending on $d$ and $\nu$ such that
    \begin{itemize}
        \item For $k$ even, if $c_1 \ne - c_{-1}$:\\ $\mu_k  \sim (c_1+c_{-1}) C(d,\nu) k^{-d-2\nu-1}$.
        \item For $k$ odd, if $c_1 \ne c_{-1}$:\\ $\mu_k \sim(c_1-c_{-1})C(d,\nu)k^{-d-2\nu-1}$.
    \end{itemize}
In the case $|c_1|=|c_{-1}|$, we have $\mu_k=o(k^{-d-2\nu+1})$ for one of the two parities (or both if $c_1=c_{-1}=0$). If $\kappa$ is infinitely differentiable on $[-1,1]$ so that no such $\nu$ exists, then $\mu_k$ decays faster than any polynomial.
\end{theorem}

The following lemmas enable us to compute the expansions of ResNTK around $\pm1$. They are proved in the supplementary material.

\begin{lemma} \label{lem:ResNTK around 1}
    For inputs in $\Spdm$ and near +1, if $\alpha > 0$ and $L \geq 1$
    \begin{equation*}
        \resntkl{L}(1-t) = 1 + c_1 t^{1/2} +o(t^{1/2})
    \end{equation*}
    where 
    \begin{equation*}
        c_1 = -\frac{1+\alpha^2L}{\sqrt{2}\pi(1+\alpha^2)}.
    \end{equation*}
\end{lemma}

\begin{lemma} \label{lem:ResNTK around m1}
    For inputs in $\Spdm$ and near -1, if $\alpha > 0$ and $L \ge 2$ then
    \begin{equation*}
        \resntkl{L}(-1+t) = 
        p_{-1}(t)+c_{-1} t^{1/2} + o(t^{1/2}),
    \end{equation*}
    with
    \begin{align*}
        |c_{-1}| \leq \frac{1}{\sqrt{2}\pi(1+\alpha^2)L}.
    \end{align*}
\end{lemma}

Lemmas \ref{lem:ResNTK around 1} and \ref{lem:ResNTK around m1} establish that for $L \ge 2$ (recall that with $L=1$ $\resntkl{1}=\ntkl{1}$) ResNTK takes the form of \eqref{eq:BBaround1} and \eqref{eq:BBaroundm1} near $\pm 1$ with $\nu=1/2$, satisfying the conditions of Theorem \ref{thm: BBdeacy}. Moreover, clearly from these lemmas
\begin{equation*}
    |c_{-1}| \leq  \frac{1}{\sqrt{2}\pi(1+\alpha^2)L} < \frac{1+\alpha^2L}{\sqrt{2}\pi(1+\alpha^2)} = |c_1|.
\end{equation*}
The eigenvalues of ResNTK, therefore, decay at the rate of $k^{-d}$ both for the odd and even frequencies, proving Theorem~\ref{thm:ResNTKdeacy}.

While the rate of decay for all frequencies is $O(k^{-d})$, the constants for the even and odd frequencies differ. In fact, if the hyperparameter $\alpha$, which relates between the residual and the skip connections, decays sufficiently fast with network depth, then the eigenvalues for the odd frequencies become extremely small compared to those for the even frequencies. This in fact happens when $\alpha$ is chosen according to \cite{Huang2020WhyDD,du2019gradient}, i.e., when $\alpha=L^{-\gamma}$ with $0.5 < \gamma \le 1$, see Figure~\ref{fig:odd_vs_even}(left). We summarize this in the following theorem.

\begin{theorem}
    With $\alpha = L ^{-\gamma}$ and $0.5 < \gamma \leq 1$, the eigenvalues of the bias-free $\resntk$ of odd frequencies $k \ge 3$ vanish.
\end{theorem}

For the proof we use the following theorem, which states that for $\alpha=L^{-\gamma}$ and $0.5 < \gamma \leq 1$, ResNTK of infinite depth converges to FC-NTK with $L=1$ hidden layer, i.e., NTK for a bias-free two-layer MLP, for which it was shown in \cite{Basri2019TheCR} that the eigenvalues for odd frequencies with $k \ge 3$ are zero. We note that this theorem, proved in the supplementary material, extends a similar theorem by \cite{Huang2020WhyDD}, who proved this only for $\gamma=1$. 
    
\begin{theorem}\label{thm:ResNTKvsMLP}
For ResNTK, as $L \rightarrow \infty$, with $\alpha = L ^{-\gamma}$, $0.5 < \gamma \leq 1$, for any two inputs $\x, \z \in \Sphere^{d-1}$, such that $1-|\x^T\z| \ge \delta > 0$ it holds that 
\begin{equation*}
    |\resntkl{L}(\x,\z)-\ntkl{1}(\x,\z)| = O(L^{1-2\gamma}).
\end{equation*}
\end{theorem}

Indeed, the convergence of ResNTK to FC-NTK with $L=1$ is also reflected in its expansion near $\pm 1$, as can be seen from the following lemma. 
\begin{lemma}
    For inputs in $\Spdm$ and near -1, if $\alpha^2 L \ll 1$ then
    \begin{equation*}
        \resntkl{L}(-1+t) = 
        c_{-1} t^{1/2} + o(t^{1/2})
    \end{equation*}
    with
    \begin{equation*}
        c_{-1} = -\frac{1}{\sqrt{2}\pi}
    \end{equation*}
\end{lemma}    
implying that when $\alpha^2L \rightarrow 0$ with $L \rightarrow \infty$ we have from Lemma \ref{lem:ResNTK around 1} that
\begin{align*}
    c_1 \xrightarrow[]{\alpha^2L \xrightarrow{} 0}
    -\frac{1}{\sqrt{2}\pi} = c_{-1}.
\end{align*}
Note that this common value of $c_1$ and $c_{-1}$ in the limit when $\alpha^2L \rightarrow 0$ is identical to the value of the coefficients in the expansion of $\ntkl{1}$ near $\pm1$ for $L=1$. 

As a consequence of Theorem \ref{thm:ResNTKvsMLP}, for a training set of $n$ samples using the Wielandt-Hoffman inequality \cite{GoluVanl96}, the eigenvalues associated with the odd frequencies are at most $O(n/L^{1-2\gamma})$.
Note that in this ResNTK differs from FC-NTK, for which in all depths except $L=1$ the eigenvalues of odd and even frequencies have similar values. Figure~\ref{fig:odd_vs_even}(left) shows the eigenvalues of ResNTK for various depth values as a function of frequency. It can be seen that as depth increases the eigenvalues of odd frequencies considerably decrease compared to those of the even frequencies.
We note finally that the difference between the odd and even frequencies disappears if we chose $\gamma=0.5$, i.e., $\alpha = 1/\sqrt{L}$, or if we include bias ($\tau>0$), as can be seen in Figure~\ref{fig:odd_vs_even}(right).

\begin{figure}[tb]
\vskip 0.2in
\begin{center}
\centerline{
\includegraphics[width=\columnwidth/2]{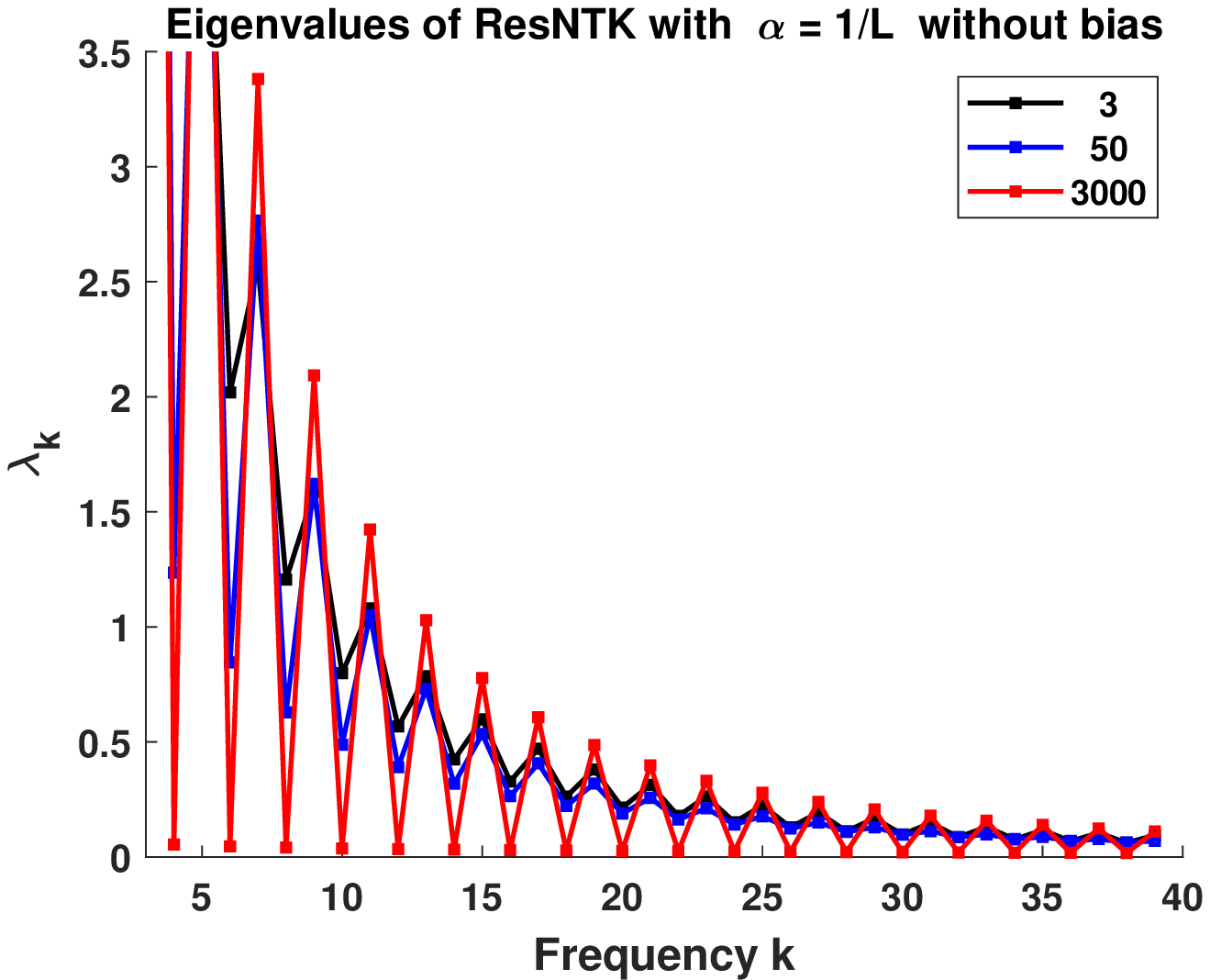} 
\includegraphics[width=\columnwidth/2]{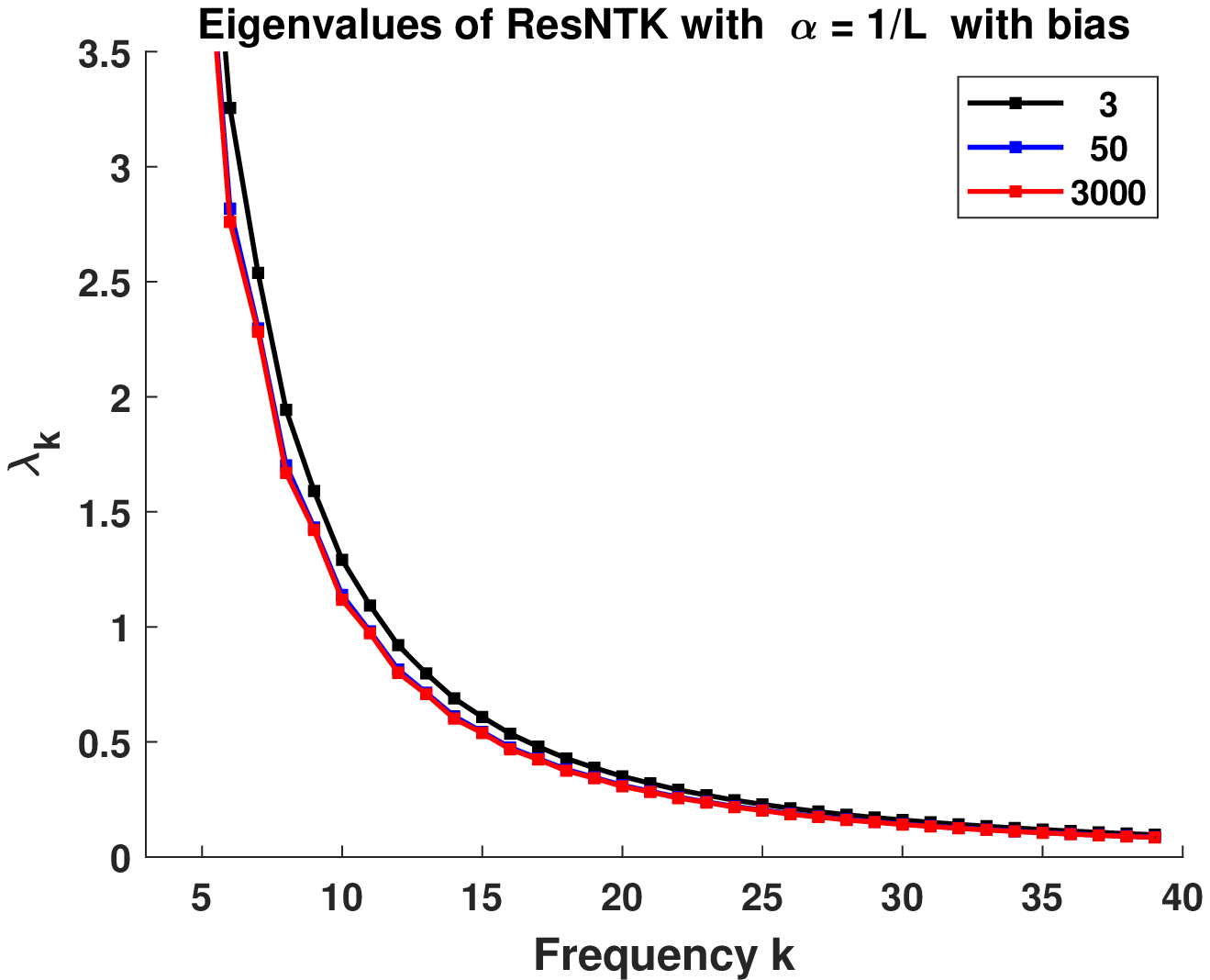} 
}

\caption{The eigenvalues of ResNTK without ($\tau=0$, left) and with bias ($\tau=1$, right) as a function of frequency for different network depths and with $\gamma=1$, i.e., $\alpha = 1/L$. With deep networks the eigenvalues of the bias-free ResNTK associated with odd frequencies ($k \ge 3$) become small, approaching 0 at $L \rightarrow \infty$. In contrast, with bias the eigenvalues decrease monotonically with frequency.} 
\label{fig:odd_vs_even}
\end{center}
\vskip -0.2in
\end{figure}

\section{Comparison of ResNTK and FC-NTK}

Theorems~\ref{thm:ResNTKeigenfunctions} and~\ref{thm:ResNTKdeacy} provide a full characterization of the set of functions in the reproducing kernel Hilbert space of ResNTK, denoted $\hh_{\resntk}$, defined in $\Spdm$ as
\begin{equation*}
    \hh_{\resntk} = \left\{f(\x) = \sum_{\substack{k\ge 0\\ \lambda_k\ne 0}}\sum_{j=1}^{N(d,k)} a_{kj} Y_{kj}(\x) \mathrm{~~s.t.~} \|f\|_{\hh_{\resntk}} < \infty \right\},
\end{equation*}
where $\lambda_k$ are the eigenvalues of $\resntk$ and
\begin{equation} \label{eq:hnorm}
    \|f\|_{\hh_{\resntk}} = \sum_{\substack{k\ge 0\\ \lambda_k\ne 0}}\sum_{j=1}^{N(d,k)} \frac{a_{kj}^2}{\lambda_k}.
\end{equation}
Our characterization of the RKHS structure of ResNTK yields similar results to those shown for FC-NTK and for the Laplace kernel \cite{bietti2020deep,chen2020laplace,geifman2020similarity}, yielding the following theorem.
\begin{theorem} \label{thm:RKHS}
Denote by $\hh_{\ntk}$ (resp.~$\bar\hh_{\ntk}$) the space of functions in the RKHS of a kernel $\ntk$ in $\Spdm$ (resp.~in $\Rd$). Then,
\begin{equation*}
    \hh_{\ntk} = \hh_{\resntk} = \hh_{\lap},
\end{equation*}
Moreover, in $\Rd$, with a radial measure (as in Thm.~\ref{thm:eig_outofsphere})
\begin{equation*}
    \bar\hh_{\ntk} = \bar\hh_{\resntk} = \bar\hh_{\hlap}.
\end{equation*}
where for $\x,\z \in \Spdm$, $\lap$ denotes the standard Laplace kernel defined by 
\begin{equation} \label{eq:laplace}
    \lap(\x,\z) = e^{-c\|\x-\z\|} = e^{-c\sqrt{2(1-\x^T\z)}},
\end{equation}
and for $\x,\z \in \Rd$ $\hlap$ is the homogenized version of the Laplace kernel, defined in \cite{geifman2020similarity} as 
\begin{equation*}
    \hlap=\|\x\|\|\z\|e^{-c\sqrt{2\left(1-\frac{\x^T\z}{\|\x\|\|\z\|}\right)}}.
\end{equation*}
\end{theorem}

A consequence of Theorem~\ref{thm:RKHS} is that the three kernels, ResNTK, FC-NTK, and the (homogenized) Laplace kernel generate functions of the same smoothness properties, i.e., all three RKHSs include functions that have weak derivatives up to order $d/2$ \cite{narcowich2007approximation}. However, the structure of the RKHSs is not identical, since every kernel is associated with a unique RKHS. Consequently, while the eigenvalues decay at the same rate, they are not identical across kernels, or even across different depths for the same kernel, producing different RKHS norms \eqref{eq:hnorm}. This, in turn, implies that when applied to the same regression problem, the kernels may produce somewhat different outcomes. For example, with deep architectures the bias-free ResNTK will be biased to interpolate functions with even frequencies, while with bias it will be agnostic to parity. Also, \cite{tirer2020kernelbased} showed that under a suitable measure, with low values of $\alpha$ ResNTK tends to produce smoother interpolations. A close examination of their experiments however reveals that also with small values of $\alpha$ their interpolations are only piecewise smooth, consistent with the structure of the respective RKHS derived here. 

Our analysis also allows to determine how sharp ResNTK is. In particular, the expansion of the Laplace kernel \eqref{eq:laplace} near 1, derived by \cite{bietti2020deep}, is given by 
\begin{equation*}
    \lap(1-t) = 1 - c\sqrt{2t} + O(t).
\end{equation*}
Therefore, the coefficient of $t^{1/2}$ indicates how steep a kernel is near 1. 
With ResNTK, its steepness depends on the choice of hyper-parameter $\alpha$, which balances between the residual and skip connections. Using Lemma~\ref{lem:ResNTK around 1} we obtain that with $c = \frac{(1+\alpha^2L)}{2\pi(1+\alpha^2)}$
\begin{equation*}
    \resntkl{L}(1-t)-\lap(1-t)=o(t^{1/2}).
\end{equation*}
Therefore, if $\alpha$ is set according to $\alpha=L^{-\gamma}$ with $0.5 \le \gamma \le 1$ then ResNTK is stable and its steepness is bounded, i.e.,
\begin{align*}
    c^{\mathrm{RES}}(L) = \frac{(1+\alpha^2L)}{2\pi(1+\alpha^2)} 
    \xrightarrow[]{L \xrightarrow{} \infty}     \begin{cases}
      \frac{1}{\pi}, & \gamma=0.5 \\
      \frac{1}{2\pi}, & 0.5 < \gamma \le 1.
    \end{cases}
\end{align*}
If however $\alpha$ is independent of depth ResNTK becomes steeper with depth. This is similar to FC-NTK, as is implied by the following lemma.
\begin{lemma}  \label{lem:c of NTK}
    With small $t>0$\footnote{Note that here we fix a slight miscalculation in \cite{bietti2020deep}(Corollary 3) which implied that the coefficient of $t^{1/2}$ is constant with depth.}
    \begin{equation} \label{eq:NTK around 1}
        \ntkl{L}(1-t) = 1 - \frac{L}{\pi\sqrt{2}}t^{1/2}+o(t^{1/2}).
    \end{equation}
    Therefore, with $c = \frac{L}{2\pi}$, $\ntkl{L}(1-t)-\lap(1-t)=o(t^{1/2})$.
\end{lemma}
Clearly therefore with deep networks FC-NTK becomes steeper near 1. This is consistent with \cite{Huang2020WhyDD} who proved that, except near $u= \x^T\z = 1$, as the depth $L$ tends to infinity FC-NTK approaches the constant 0.25. Therefore with deep architectures FC-NTK forms a spike.

Figure~\ref{fig:NTKvsL} shows the shape of both FC-NTK and ResNTK for three choices of network depths. Our experiments (Section~\ref{sec:expreiments}) indeed show that for FC-NTK and ResNTK with constant value of $\alpha$ learning accuracy degrades with depth, while with a decaying $\alpha$ learning accuracy is stable across depth.

\begin{figure}[tb]
\vskip 0.2in
\centerline{
\includegraphics[width=\columnwidth/3]{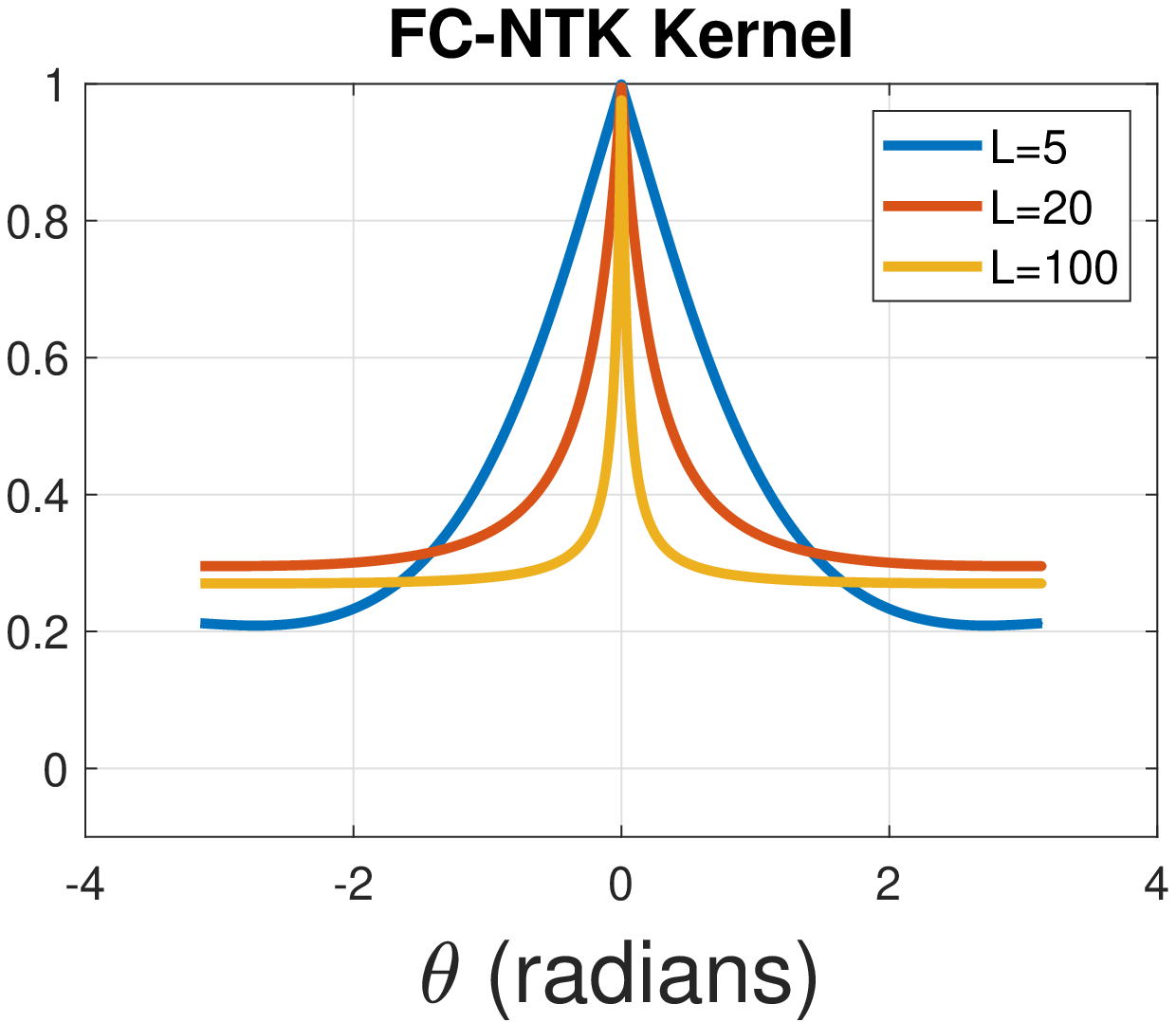}
\includegraphics[width=\columnwidth/3]{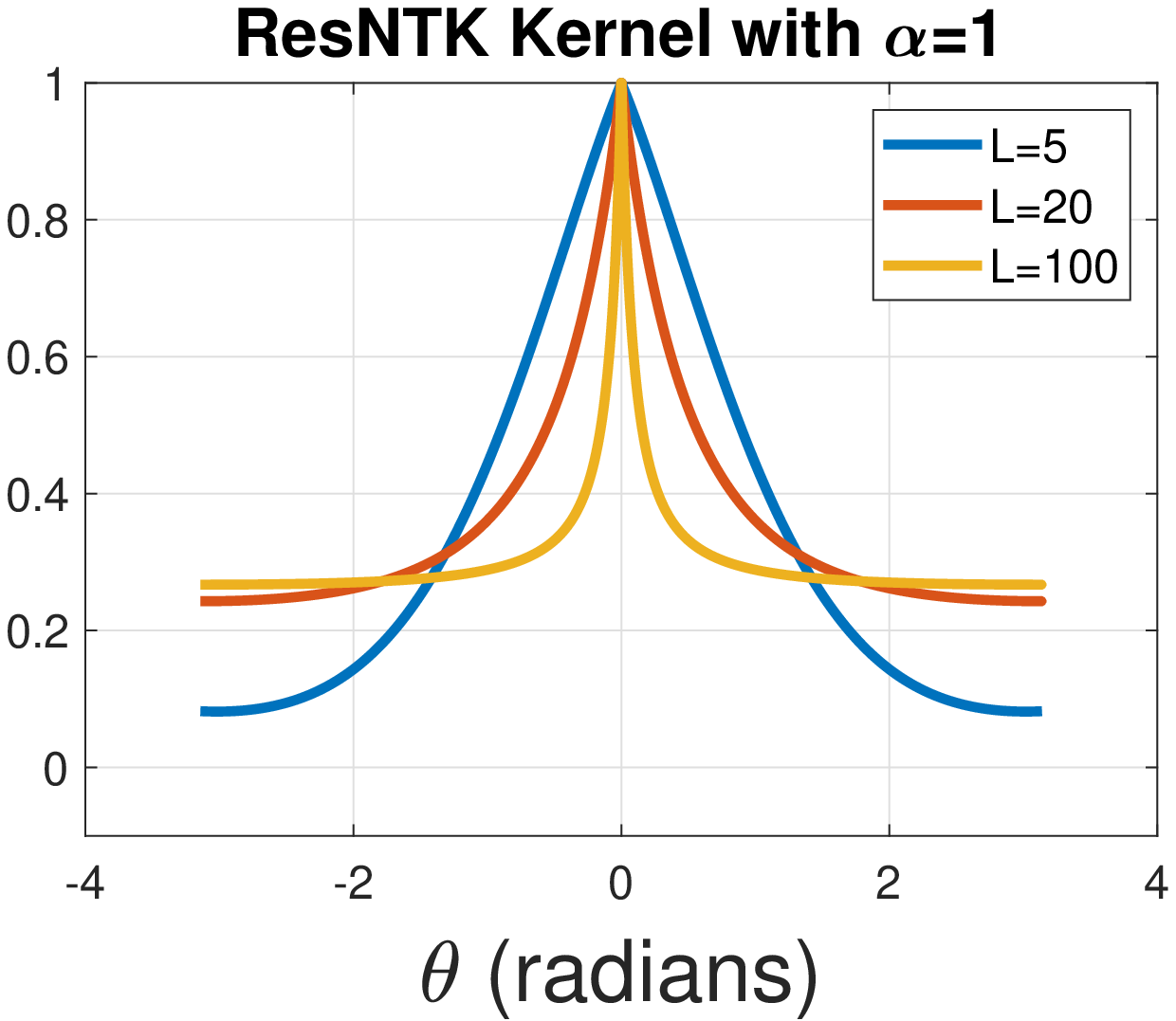}
\includegraphics[width=\columnwidth/3]{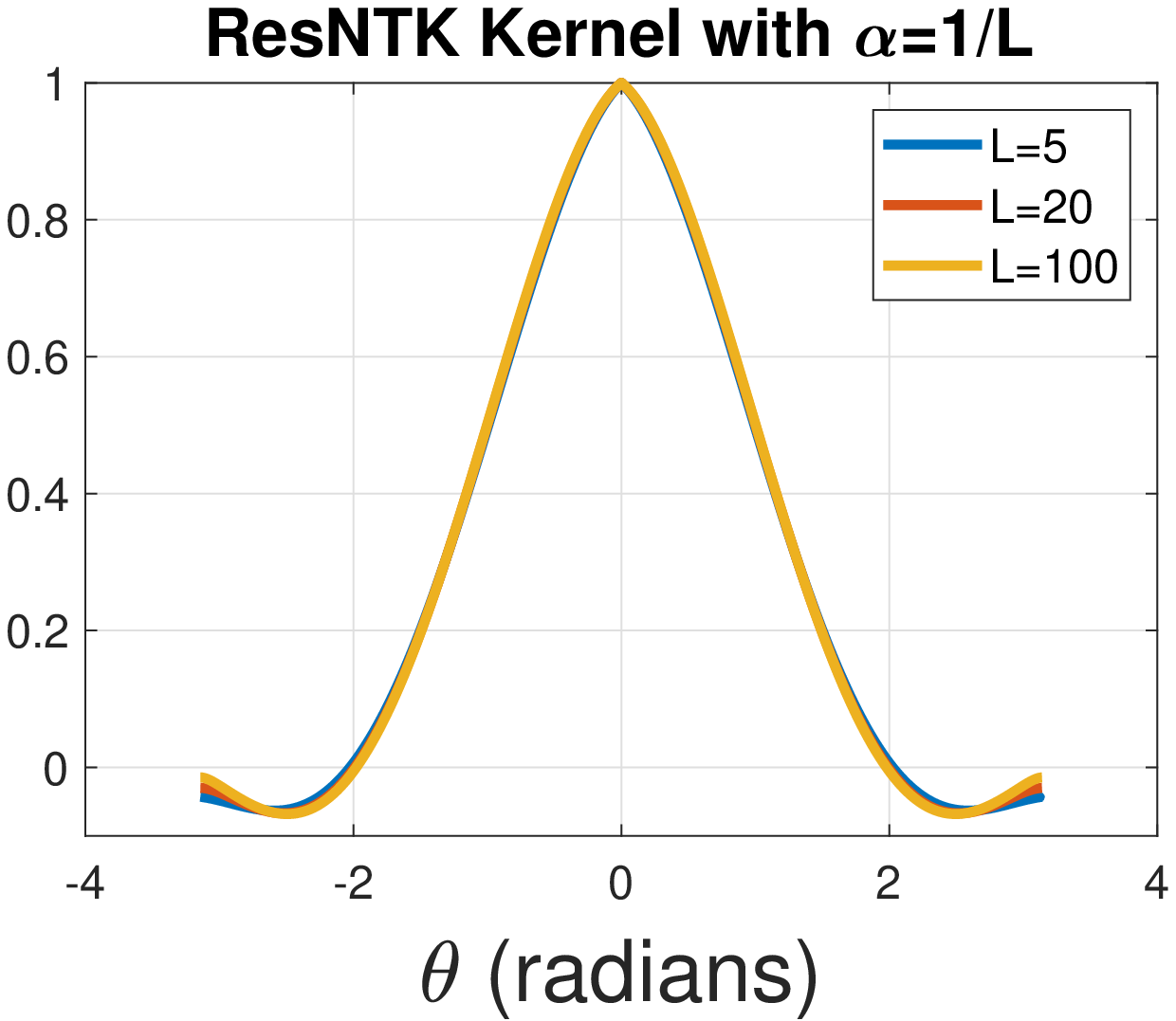}
}
\caption{FC-NTK (left) and ResNTK (center $\alpha=1$, right $\alpha=1/L$) for networks of different depths, $L=5,20,100$. For FC-NTK and ResNTK with $\alpha=1$, the kernel becomes spiky with depth. With $\alpha=1/L$ ResNTK remains stable for all depths.}
\label{fig:NTKvsL}
\vskip -0.2in
\end{figure}

\begin{table*}[t]
\caption{Classification accuracies on the UCI dataset obtained by applying FC-NTK and ResNTK with $\alpha\in \{1/L,1/\sqrt{L}, 1\}$.}
\label{tab:KernelsonUCI}
\vskip 0.15in
\begin{center}
\begin{tiny}
\begin{sc}
\begin{tabular}[t]{ccccr}
\toprule
Number of Layers & FC-NTK & ResNTK, $\alpha = \frac{1}{L}$ & ResNTK, $\alpha=\frac{1}{\sqrt{L}}$ & ResNTK, $\alpha = 1$ \\
\midrule
5   & 85.54 $\pm$ 10.70 & 85.59 $\pm$ 10.61  & 85.52 $\pm$ 10.95 & 86.02 $\pm$ 9.660\\
25  & 84.28 $\pm$ 11.18 & 85.51 $\pm$ 10.82  & 85.46 $\pm$ 10.69 & 85.21 $\pm$ 10.10\\
50  & 82.97 $\pm$ 11.44 & 85.45 $\pm$ 10.80  & 85.25 $\pm$ 10.86 & 79.94 $\pm$ 16.55\\
100 & 80.87 $\pm$ 12.08 & 85.38 $\pm$ 10.75  & 84.86 $\pm$ 10.93 & 79.91 $\pm$ 16.10\\
\bottomrule
\end{tabular}
\end{sc}
\end{tiny}
\end{center}
\vskip -0.1in
\end{table*}

\begin{table}[h]
\caption{Classification accuracies on the CIFAR-10 dataset obtained by applying FC-NTK and ResNTK with $\alpha\in \{1/L,1/\sqrt{L}\}$.}
\label{tab:KernelsonCIFAR}
\vskip 0.15in
\begin{center}
\begin{tiny}
\begin{sc}
\begin{tabular}{ccccr}
\toprule
Number of Layers & FC-NTK & ResNTK, $\alpha = \frac{1}{L}$ & ResNTK, $\alpha=\frac{1}{\sqrt{L}}$ \\
\midrule
5   & 58.29 & 58.23 &58.32   \\
25  &54.33  & 57.72 &  58.33 \\
50  & 51.42 & 57.58 & 58.34   \\
100 & 48.27 & 57.53 &58.34   \\
\bottomrule
\end{tabular}
\end{sc}
\end{tiny}
\end{center}
\vskip -0.1in
\end{table}

\begin{table}[tb]
\caption{Classification accuracies on the SVHN dataset obtained by applying FC-NTK and ResNTK with $\alpha\in \{1/L,1/\sqrt{L}\}$.
}
\label{tab:KernelsonSVHN}
\vskip 0.15in
\begin{center}
\begin{tiny}
\begin{sc}
\begin{tabular}{cccr}
\toprule
Number of Layers & FC-NTK &  ResNTK, $\alpha = \frac{1}{L}$ & ResNTK, $\alpha=\frac{1}{\sqrt{L}}$ \\
\midrule
5   & 74.44  & 73.62 & 78.36  \\
25  & 48.75  &  74.73 & 78.17 \\
50  & 33.69  & 74.89 & 78.14   \\
100 & 21.12 & 74.91 & 78.13   \\
\bottomrule
\end{tabular}
\end{sc}
\end{tiny}
\end{center}
\vskip -0.1in
\end{table}

\section{Experiments} \label{sec:expreiments}

We performed a number of experiments to show the effect of depth on ResNTK and to compare it to FC-NTK.

\textbf{UCI Dataset} We applied ResNTK and FC-NTK to 90 datasets of the UCI collection ($<5000$ items) using the protocol of \cite{Arora2020Harnessing}. We applied ridge regression with smoothness constant $\lambda=1e^{-3}$ and normalized each data item to unit norm. To solve a classification problem, for each test item we regress each kernel to a one-hot vector and select the class that maximizes the regression result. For ResNTK we used a decaying balancing parameter ($\alpha=1/L,1/\sqrt{L}$) as well as constant $\alpha=1$. (Due to condition number problems, in the case of constant $\alpha$ we only report results for 63 datasets.) We report average classification accuracy. Table~\ref{tab:KernelsonUCI} shows average accuracy for different depth values. It can be seen that while FC-NTK and ResNTK with $\alpha=1$ degrade with depth, from roughly 86\% with 5 hidden layers to 80-81\% with 100 layers, ResNTK with $\alpha=1/L$ and $\alpha=1/\sqrt{L}$ remain stable around 85-85.5\%. Interestingly, in the latter cases also the standard deviations remain stable across different depths. We note that these results, peaked for FC-NTK at 85.54\%, are comparable to those shown in \cite{Arora2020Harnessing}, who reported an average accuracy of 81.95\% on 90 datasets with hyper-parameter search, including depth and testing also with a Gaussian Process kernel.

\textbf{CIFAR-10} We next applied both kernels to the CIFAR-10 dataset. Note that the kernels we applied correspond to classical and residual fully connected architectures and are not convolutional. We normalized the pixels in each image to zero mean and unit variance and used kernel regression with $\lambda=0$. Table~\ref{tab:KernelsonCIFAR} shows classification accuracies with FC-NTK and ResNTK with $\alpha \in \{1/L,1/\sqrt{L}\}$. As with the UCI experiments, test accuracies for FC-NTK degrade from 58.28\% for 5 layers to 48.27\% for 100 layers. In contrast, ResNTK with $\alpha \in \{1/L,1/\sqrt{L}\}$ maintains an accuracy of 57.5\%-58.3\% across depth.

\textbf{SVHN} We repeated the same experiments on the SVHN dataset, see Table~\ref{tab:KernelsonSVHN}. Here too we normalized the pixels in each image to zero mean and unit variance but used regression with $\lambda=1e^{-5}$. The differences between FC-NTK and ResNTK are even more extreme in this experiment. FC-NTK degrades from an accuracy of 74.44\% with 5 layers to 21.12\% with 100 layers, while ResNTK with $\alpha=1/L$ and $\alpha=1/\sqrt{L}$ maintains respectively a 74-75\% and 78\% accuracy for all tested depths.

\comment{

\textbf{CIFAR-10: Real networks} To further examine the relevance of these results to real networks we applied a network with fully connected (FCN) and residual, fully connected network with hidden layers of width $m=2000$. \rb{Yuval?} We optimized the networks with SGD with no momentum. We present the results of these experiments obtained with two choices of learning rates which were kept constant, a large rate of $1e^{-2}$ and a smaller rate of $1e^{-4}$. Presumably the smaller rate better approximates the NTK regime of networks. The results can be seen in Tables~\ref{tab:NNonCIFAR_large_lr}-\ref{tab:NNonCIFAR_small_lr}. With the large learning rate all networks achieved a similar accuracy of 0.53-0.54, except for the fully connected network, which failed to converge with deep network architectures (100 and 200 layers). With the smaller learning rate the pattern of results is similar to those that were achieved with the respective kernels. Specifically, FCN degrades from 0.51 with 5 layers to 0.39 at deeper layers, whereas ResNet with $\gamma\in\{1,0.5\}$ maintains an accuracy of 0.52-0.53.

\begin{table}[t]
\caption{Classification accuracies on CIFAR-10 with large learning rate $\eta=1e^{-2}$.}
\label{tab:NNonCIFAR_large_lr}
\vskip 0.15in
\begin{center}
\begin{tiny}
\begin{sc}
\begin{tabular}{ccccr}
\toprule
Number of Layers & FCN & ResNet $\gamma=1$ & ResNet $\gamma=0.5$ \\
\midrule
5   & 0.53 & 0.53 & 0.54  \\
25  & 0.53 & 0.53 & 0.54  \\
50  & 0.53 & 0.54 & 0.54  \\
100 & 0.2  & 0.54 & 0.54  \\
200 & 0.1  & 0.54 & 0.53  \\
\bottomrule
\end{tabular}
\end{sc}
\end{tiny}
\end{center}
\vskip -0.1in
%
\caption{Classification accuracies on CIFAR-10 with small learning rate $\eta=1e^{-4}$.}
\label{tab:NNonCIFAR_small_lr}
\vskip 0.15in
\begin{center}
\begin{tiny}
\begin{sc}
\begin{tabular}{ccccr}
\toprule
Number of Layers & FCN & ResNet $\gamma=1$ & ResNet $\gamma=0.5$ \\
\midrule
5   & 0.51 & 0.53 & 0.52  \\
25  & 0.42 & 0.53 & 0.52  \\
50  & 0.39 & 0.53 & 0.52  \\
100 & 0.39 & 0.52 & 0.52  \\
200 & 0.39 & 0.52 & 0.52  \\
\bottomrule
\end{tabular}
\end{sc}
\end{tiny}
\end{center}
\vskip -0.1in
\end{table}

}

\section{Conclusion}

We have provided derivations to determine the RKHS structure of NTK for residual networks. Our analysis indicates that, similar to NTK for classical, fully connected networks, the eigenfunctions of ResNTK are the (scaled) spherical harmonics and its eigenvalues decay polynomially with frequency $k$ at the rate of $k^{-d}$. These in turn imply that the set of functions in its RKHS are identical to those of both FC-NTK and the Laplace kernel restricted to the hypersphere $\Spdm$. Our results imply that all three kernels produce functions of similar smoothness properties. We however showed that depending on the choice of $\alpha$, which balances between the residual and skip connections, ResNTK can be controlled to become spiky with depth, as is the case with FC-NTK, or maintain a stable shape. In addition, we showed that deep bias-free ResNTK is significantly biased toward the even frequencies.

Our results suggest that NTK provides only a partial explanation to the success of residual networks. Indeed it appears that classification with FC-NTK degrades with depth, while classification with ResNTK can be made stable with a proper choice of a balancing hyper-parameter. However, our experiments suggest that with an optimal choice of depth classification results with FC-NTK and ResNTK are similar, most likely due to their similar RKHS structures. This is somewhat in contrast to actual implementations in which residual networks seem to significantly outperform classical feed-forward networks. This difference may be attributed to optimization issues, or to the possible invalidity of the assumptions of NTK to real networks of finite width. It is also possible that differences between residual and classical kernels are more significant in convolutional architectures.

\bibliography{main_arxiv}

\begin{thebibliography}{38}
\providecommand{\natexlab}[1]{#1}
\providecommand{\url}[1]{\texttt{#1}}
\expandafter\ifx\csname urlstyle\endcsname\relax
  \providecommand{\doi}[1]{doi: #1}\else
  \providecommand{\doi}{doi: \begingroup \urlstyle{rm}\Url}\fi

\bibitem[Allen-Zhu et~al.(2019)Allen-Zhu, Li, and Song]{allen2019convergence}
Allen-Zhu, Z., Li, Y., and Song, Z.
\newblock A convergence theory for deep learning via over-parameterization.
\newblock In \emph{International Conference on Machine Learning}, pp.\
  242--252. PMLR, 2019.

\bibitem[Arora et~al.(2019)Arora, Du, Hu, Li, and Wang]{arora2019finegrained}
Arora, S., Du, S., Hu, W., Li, Z., and Wang, R.
\newblock Fine-grained analysis of optimization and generalization for
  overparameterized two-layer neural networks.
\newblock In \emph{International Conference on Machine Learning}, pp.\
  322--332. PMLR, 2019.

\bibitem[Arora et~al.(2020)Arora, Du, Li, Salakhutdinov, Wang, and
  Yu]{Arora2020Harnessing}
Arora, S., Du, S.~S., Li, Z., Salakhutdinov, R., Wang, R., and Yu, D.
\newblock Harnessing the power of infinitely wide deep nets on small-data
  tasks.
\newblock In \emph{International Conference on Learning Representations}, 2020.

\bibitem[Balduzzi et~al.(2017)Balduzzi, Frean, Leary, Lewis, Ma, and
  McWilliams]{balduzzi2017shattered}
Balduzzi, D., Frean, M., Leary, L., Lewis, J., Ma, K. W.-D., and McWilliams, B.
\newblock The shattered gradients problem: If resnets are the answer, then what
  is the question?
\newblock \emph{arXiv preprint arXiv:1702.08591}, 2017.

\bibitem[Basri et~al.(2019)Basri, Jacobs, Kasten, and
  Kritchman]{Basri2019TheCR}
Basri, R., Jacobs, D.~W., Kasten, Y., and Kritchman, S.
\newblock The convergence rate of neural networks for learned functions of
  different frequencies.
\newblock In Wallach, H.~M., Larochelle, H., Beygelzimer, A., d'Alché Buc, F.,
  Fox, E.~A., and Garnett, R. (eds.), \emph{Advances in Neural Information
  Pro-cessing Systems}, pp.\  4763--4772, 2019.

\bibitem[Basri et~al.(2020)Basri, Galun, Geifman, Jacobs, Kasten, and
  Kritchman]{basri2020frequency}
Basri, R., Galun, M., Geifman, A., Jacobs, D., Kasten, Y., and Kritchman, S.
\newblock Frequency bias in neural networks for input of non-uniform density.
\newblock In \emph{International Conference on Machine Learning}, pp.\
  685--694. PMLR, 2020.

\bibitem[Bietti \& Bach(2020)Bietti and Bach]{bietti2020deep}
Bietti, A. and Bach, F.
\newblock Deep equals shallow for relu networks in kernel regimes.
\newblock \emph{arXiv preprint arXiv:2009.14397}, 2020.

\bibitem[Bietti \& Mairal(2019)Bietti and Mairal]{bietti2019inductive}
Bietti, A. and Mairal, J.
\newblock On the inductive bias of neural tangent kernels.
\newblock In \emph{Advances in Neural Information Processing Systems}, pp.\
  12893--12904, 2019.

\bibitem[Cao et~al.(2019)Cao, Fang, Wu, Zhou, and Gu]{cao2019}
Cao, Y., Fang, Z., Wu, Y., Zhou, D.-X., and Gu, Q.
\newblock Towards understanding the spectral bias of deep learning.
\newblock \emph{arXiv preprint arXiv:2009.01198}, 2019.

\bibitem[Chen \& Xu(2020)Chen and Xu]{chen2020laplace}
Chen, L. and Xu, S.
\newblock Deep neural tangent kernel and laplace kernel have the same rkhs.
\newblock \emph{arXiv preprint arXiv:2009.10683}, 2020.

\bibitem[Chizat et~al.(2019)Chizat, Oyallon, and Bach]{chizat2019lazy}
Chizat, L., Oyallon, E., and Bach, F.
\newblock On lazy training in differentiable programming.
\newblock In \emph{Advances in Neural Information Processing Systems}, pp.\
  2937--2947, 2019.

\bibitem[Cho \& Saul(2009)Cho and Saul]{Cho2009NIPS}
Cho, Y. and Saul, L.
\newblock Kernel methods for deep learning.
\newblock In Bengio, Y., Schuurmans, D., Lafferty, J., Williams, C., and
  Culotta, A. (eds.), \emph{Advances in Neural Information Processing Systems},
  volume~22, pp.\  342--350. Curran Associates, Inc., 2009.

\bibitem[Du et~al.(2019)Du, Lee, Li, Wang, and Zhai]{du2019gradient}
Du, S., Lee, J., Li, H., Wang, L., and Zhai, X.
\newblock Gradient descent finds global minima of deep neural networks.
\newblock In \emph{International Conference on Machine Learning}, pp.\
  1675--1685. PMLR, 2019.

\bibitem[Gallier(2009)]{gallier2009notes}
Gallier, J.
\newblock Notes on spherical harmonics and linear representations of lie
  groups.
\newblock 2009.

\bibitem[Geifman et~al.(2020)Geifman, Yadav, Kasten, Galun, Jacobs, and
  Basri]{geifman2020similarity}
Geifman, A., Yadav, A., Kasten, Y., Galun, M., Jacobs, D., and Basri, R.
\newblock On the similarity between the laplace and neural tangent kernels.
\newblock \emph{arXiv preprint arXiv:2007.01580}, 2020.

\bibitem[Golub \& Van~Loan(1996)Golub and Van~Loan]{GoluVanl96}
Golub, G.~H. and Van~Loan, C.~F.
\newblock \emph{Matrix Computations}.
\newblock The Johns Hopkins University Press, third edition, 1996.

\bibitem[Greenfeld et~al.(2019)Greenfeld, Galun, Basri, Yavneh, and
  Kimmel]{greenfeld2019amg}
Greenfeld, D., Galun, M., Basri, R., Yavneh, I., and Kimmel, R.
\newblock Learning to optimize multigrid {PDE} solvers.
\newblock In Chaudhuri, K. and Salakhutdinov, R. (eds.), \emph{Proceedings of
  the 36th International Conference on Machine Learning}, volume~97, pp.\
  2415--2423, 2019.

\bibitem[He et~al.(2016{\natexlab{a}})He, Zhang, Ren, and Sun]{he2016deep}
He, K., Zhang, X., Ren, S., and Sun, J.
\newblock Deep residual learning for image recognition.
\newblock In \emph{Proceedings of the IEEE conference on computer vision and
  pattern recognition}, pp.\  770--778, 2016{\natexlab{a}}.

\bibitem[He et~al.(2016{\natexlab{b}})He, Zhang, Ren, and Sun]{he2016identity}
He, K., Zhang, X., Ren, S., and Sun, J.
\newblock Identity mappings in deep residual networks.
\newblock In Leibe, B., Matas, J., Sebe, N., and Welling, M. (eds.),
  \emph{Computer Vision -- ECCV 2016}. Springer International Publishing,
  2016{\natexlab{b}}.

\bibitem[Howard et~al.(2019)Howard, Sandler, Chu, Chen, Chen, Tan, Wang, Zhu,
  Pang, Vasudevan, et~al.]{howard2019searching}
Howard, A., Sandler, M., Chu, G., Chen, L.-C., Chen, B., Tan, M., Wang, W.,
  Zhu, Y., Pang, R., Vasudevan, V., et~al.
\newblock Searching for mobilenetv3.
\newblock In \emph{Proceedings of the IEEE International Conference on Computer
  Vision}, pp.\  1314--1324, 2019.

\bibitem[Huang et~al.(2020)Huang, Wang, Tao, and Zhao]{Huang2020WhyDD}
Huang, K., Wang, Y., Tao, M., and Zhao, T.
\newblock Why do deep residual networks generalize better than deep feed
  forward networks? - a neural tangent kernel perspective.
\newblock \emph{ArXiv}, abs/2002.06262, 2020.

\bibitem[Jacot et~al.(2018)Jacot, Gabriel, and Hongler]{jacot2018}
Jacot, A., Gabriel, F., and Hongler, C.
\newblock Neural tangent kernel: Convergence and generalization in neural
  networks.
\newblock In Bengio, S., Wallach, H., Larochelle, H., Grauman, K.,
  Cesa-Bianchi, N., and Garnett, R. (eds.), \emph{Advances in Neural
  Information Processing Systems 31}, pp.\  8571--8580. 2018.

\bibitem[Lee et~al.(2020)Lee, Schoenholz, Pennington, Adlam, Xiao, Novak, and
  Sohl{-}Dickstein]{lee2020finite}
Lee, J., Schoenholz, S.~S., Pennington, J., Adlam, B., Xiao, L., Novak, R., and
  Sohl{-}Dickstein, J.
\newblock Finite versus infinite neural networks: an empirical study.
\newblock In Larochelle, H., Ranzato, M., Hadsell, R., Balcan, M., and Lin, H.
  (eds.), \emph{Advances in Neural Information Processing Systems 33}, 2020.

\bibitem[Li et~al.(2018)Li, Xu, Taylor, Studer, and
  Goldstein]{li2018visualizing}
Li, H., Xu, Z., Taylor, G., Studer, C., and Goldstein, T.
\newblock Visualizing the loss landscape of neural nets.
\newblock In \emph{Advances in neural information processing systems}, pp.\
  6389--6399, 2018.

\bibitem[Liang et~al.(2019)Liang, Rakhlin, and Zhai]{liang2019risk}
Liang, T., Rakhlin, A., and Zhai, X.
\newblock On the risk of minimum-norm interpolants and restricted lower
  isometry of kernels.
\newblock \emph{arXiv preprint arXiv:1908.10292}, 2019.

\bibitem[Liang et~al.(2020)Liang, Rakhlin, et~al.]{liang2020just}
Liang, T., Rakhlin, A., et~al.
\newblock Just interpolate: Kernel “ridgeless” regression can generalize.
\newblock \emph{Annals of Statistics}, 48\penalty0 (3):\penalty0 1329--1347,
  2020.

\bibitem[Liu et~al.(2019)Liu, Chen, Zhou, Du, Zhou, and Zhao]{liu2019towards}
Liu, T., Chen, M., Zhou, M., Du, S.~S., Zhou, E., and Zhao, T.
\newblock Towards understanding the importance of shortcut connections in
  residual networks.
\newblock In \emph{Advances in neural information processing systems}, pp.\
  7892--7902, 2019.

\bibitem[Narcowich et~al.(2007)Narcowich, Sun, and
  Ward]{narcowich2007approximation}
Narcowich, F.~J., Sun, X., and Ward, J.~D.
\newblock Approximation power of rbfs and their associated sbfs: a connection.
\newblock \emph{Advances in Computational Mathematics}, 27\penalty0
  (1):\penalty0 107--124, 2007.

\bibitem[Pagliana et~al.(2020)Pagliana, Rudi, De~Vito, and
  Rosasco]{pagliana2020interpolation}
Pagliana, N., Rudi, A., De~Vito, E., and Rosasco, L.
\newblock Interpolation and learning with scale dependent kernels.
\newblock \emph{arXiv preprint arXiv:2006.09984}, 2020.

\bibitem[Radosavovic et~al.(2020)Radosavovic, Kosaraju, Girshick, He, and
  Doll{\'a}r]{radosavovic2020designing}
Radosavovic, I., Kosaraju, R.~P., Girshick, R., He, K., and Doll{\'a}r, P.
\newblock Designing network design spaces.
\newblock In \emph{Proceedings of the IEEE/CVF Conference on Computer Vision
  and Pattern Recognition}, pp.\  10428--10436, 2020.

\bibitem[Rahaman et~al.(2019)Rahaman, Baratin, Arpit, Draxler, Lin, Hamprecht,
  Bengio, and Courville]{rahaman2019spectral}
Rahaman, N., Baratin, A., Arpit, D., Draxler, F., Lin, M., Hamprecht, F.,
  Bengio, Y., and Courville, A.
\newblock On the spectral bias of neural networks.
\newblock In Chaudhuri, K. and Salakhutdinov, R. (eds.), \emph{Proceedings of
  the 36th International Conference on Machine Learning}, volume~97 of
  \emph{Proceedings of Machine Learning Research}, pp.\  5301--5310. PMLR,
  2019.

\bibitem[{Siravenha} et~al.(2019){Siravenha}, {Reis}, {Cordeiro}, {Tourinho},
  {Gomes}, and {Carvalho}]{Siravenha2019resmlp}
{Siravenha}, A. C.~Q., {Reis}, M. N.~F., {Cordeiro}, I., {Tourinho}, R.~A.,
  {Gomes}, B.~D., and {Carvalho}, S.~R.
\newblock Residual mlp network for mental fatigue classification in mining
  workers from brain data.
\newblock In \emph{2019 8th Brazilian Conference on Intelligent Systems
  (BRACIS)}, pp.\  407--412, 2019.
\newblock \doi{10.1109/BRACIS.2019.00078}.

\bibitem[Tan et~al.(2019)Tan, Chen, Pang, Vasudevan, Sandler, Howard, and
  Le]{tan2019mnasnet}
Tan, M., Chen, B., Pang, R., Vasudevan, V., Sandler, M., Howard, A., and Le,
  Q.~V.
\newblock Mnasnet: Platform-aware neural architecture search for mobile.
\newblock In \emph{Proceedings of the IEEE Conference on Computer Vision and
  Pattern Recognition}, pp.\  2820--2828, 2019.

\bibitem[Tirer et~al.(2020)Tirer, Bruna, and Giryes]{tirer2020kernelbased}
Tirer, T., Bruna, J., and Giryes, R.
\newblock Kernel-based smoothness analysis of residual networks.
\newblock \emph{arXiv preprint arXiv:2009.10008}, 2020.

\bibitem[Veit et~al.(2016)Veit, Wilber, and Belongie]{veit2016residual}
Veit, A., Wilber, M.~J., and Belongie, S.
\newblock Residual networks behave like ensembles of relatively shallow
  networks.
\newblock \emph{Advances in neural information processing systems},
  29:\penalty0 550--558, 2016.

\bibitem[Xu et~al.(2019)Xu, Zhang, Luo, Xiao, and Ma]{Xu2019}
Xu, Z.~J., Zhang, Y., Luo, T., Xiao, Y., and Ma, Z.
\newblock Frequency principle: Fourier analysis sheds light on deep neural
  networks.
\newblock \emph{CoRR}, abs/1901.06523, 2019.

\bibitem[Zhang et~al.(2019{\natexlab{a}})Zhang, Dauphin, and
  Ma]{zhang2019residual}
Zhang, H., Dauphin, Y.~N., and Ma, T.
\newblock Residual learning without normalization via better initialization.
\newblock In \emph{International Conference on Learning Representations},
  2019{\natexlab{a}}.

\bibitem[Zhang et~al.(2019{\natexlab{b}})Zhang, Yu, Yi, Chen, and
  Liu]{zhang2019stability}
Zhang, H., Yu, D., Yi, M., Chen, W., and Liu, T.-y.
\newblock Stability and convergence theory for learning resnet: A full
  characterization.
\newblock \emph{arXiv preprint arXiv:1903.07120}, 2019{\natexlab{b}}.

\end{thebibliography}
\bibliographystyle{icml2021}

\onecolumn

\section*{Appendix}

\appendix

\section{Eigenfunctions of ResNTK}
We next prove Theorem 4.1 from the paper.
\begin{theorem}
    Bias-free ResNTK is homogeneous of degree 1 and zonal, i.e., $\resntk(\x,\z)=\|\x\|\|\z\|\resntk\left(\frac{\x^T\z}{\|\x\|\|\z\|}\right)$. Its eigenfunctions under the uniform measure in $\Spdm$ are the spherical harmonics.
\end{theorem}
\begin{proof}
    We use the notation for $\resntk(\x,\z)$ defined in Section 3.2 in the paper, without bias, i.e., $\tau=0$. We first show that for all $\ell \in \{0,...,L-1\}$ $K_\ell$ is homogeneous of degree 1 and zonal (abbreviated H1Z), i.e., for $\x, \z \in \Real^d$
    \begin{equation}
    \label{eq:k_resnet_l}
    K_\ell(\x,\z) = 
    \norm{\x}\norm{\z}K_{\ell}\left(\frac{\x^T \z}{\norm{\x}\norm{\z}}\right).
    \end{equation}
    First, clearly $K_0(\x,\z)=\x^T\z$ is H1Z. Next, suppose $K_\ell$ is H1Z, then
    \begin{align*}
        v_\ell(\x,\z) &= \sqrt{K_\ell(\x,\x)K_\ell(\z,\z)}=\|\x\|\|\z\|K_\ell(1)\\
        u_\ell(\x,\z) &= \frac{K_\ell(\x,\z)}{v_\ell(\x,\z)} = \frac{K_\ell\left(\frac{\x^T\z}{\|\x\|\|\z\|}\right)}{K_\ell(1)} \\
        K_{\ell+1}(\x,\z) &= K_{\ell}(\x,\z) + \alpha^2 v_\ell(\x,\z) \kappa_1(u_\ell(\x,\z)) \\ 
        &= \|\x\|\|\z\| \left( K_\ell\left(\frac{\x^T\z}{\|\x\|\|\z\|}\right) + \alpha^2 K_\ell(1) \kappa_1 \left( \frac{K_\ell\left(\frac{\x^T\z}{\|\x\|\|\z\|}\right)}{K_\ell(1)} \right) \right) 
        = \|\x\|\|\z\| K_{\ell+1}\left(\frac{\x^T\z}{\|\x\|\|\z\|}\right),
    \end{align*}
    implying that $K_{\ell+1}$ is H1Z.
    
    Next, we show that $B_\ell$ is homogeneous of degree 0 and zonal (abbreviated H0Z), i.e.,
    \begin{equation}
    \label{eq:B_resnet_ll}
    B_{\ell+1}\left(\x,\z\right) 
    = B_{\ell+1}\left(\frac{\x^T \z}{\norm{\x}\norm{\z}}\right).
    \end{equation} 
    $B_{L+1}(\x,\z) = 1$ is trivially H0Z. Suppose $B_{\ell+1}$ is H0Z, then
    \begin{align*}
        B_\ell(\x,\z) = B_{\ell+1}(\x,\z) [1 + \alpha^2 \kappa_0(u_{\ell-1})] = B_{\ell+1}\left(\frac{\x^T \z}{\norm{\x}\norm{\z}}\right)\left[1+\alpha^2\kappa_0\left(\frac{K_\ell\left(\frac{\x^T\z}{\|\x\|\|\z\|}\right)}{K_\ell(1)}\right)\right]=B_\ell\left(\frac{\x^T\z}{\|\x\|\|\z\|}\right).
    \end{align*}

    Finally, using \eqref{eq:k_resnet_l} and \eqref{eq:B_resnet_ll}
    \begin{align*}
        \resntkl{L}(\x,\z) &= C \sum_{\ell=1}^L B_{\ell+1}(\x,\z)[v_{\ell-1}(\x,\z)\kappa_1(u_{\ell-1}(\x,\z)) + K_{\ell-1}(\x,\z)\kappa_0(u_{\ell-1}(\x,\z))] \\
        &= C \sum_{\ell=1}^L B_{\ell+1}\left(\frac{\x^T\z}{\|\x\|\|\z\|}\right) \|\x\|\|\z\| \left[K_{\ell-1}(1) \kappa_1\left(\frac{K_\ell\left(\frac{\x^T\z}{\|\x\|\|\z\|}\right)}{K_\ell(1)}\right) + K_{\ell-1}\left(\frac{\x^T\z}{\|\x\|\|\z\|}\right) \kappa_0\left(\frac{K_\ell\left(\frac{\x^T\z}{\|\x\|\|\z\|}\right)}{K_\ell(1)}\right) \right] \\
        &= \|\x\|\|\z\| \resntkl{L}\left(\frac{\x^T\z}{\|\x\|\|\z\|}\right).
    \end{align*}
    Consequently, $\resntkl{L}$ is homogeneous of degree 1 and zonal, and therefore, with the uniform measure in $\Spdm$ the eigenfunctions of $\resntkl{L}$ are the spherical harmonics.
\end{proof}

\section{Decay rate of ResNTK}
In this section, we prove Lemmas 4.5, 4.6 and 4.9. We start with supporting Lemmas and notations that we use in this section.
\begin{lemma} \label{lem:norm factor Rd} \cite{Huang2020WhyDD}
    For every $\x \in \Rd$, $K_{\ell}(\x,\x) = \norm{\x}^2(1+\alpha^2)^{\ell}$.
\end{lemma}
\begin{proof}
    With $\ell=0$, $K_{0}(\x,\x) = \x^T\x = \norm{\x}^2 = \norm{\x}^2(1+\alpha^2)^0$. Using the recursive definition of $K_\ell$,
    \begin{equation*}
        K_{\ell}(\x,\x) = K_{\ell-1}(\x,\x)+\alpha^2K_{\ell-1}(\x,\x) \kappa_1\left(\frac{K_{\ell-1}(\x,\x)}{\sqrt{K_{\ell-1}(\x,\x)K_{\ell-1}(\x,\x)}}\right) = K_{\ell-1}(\x,\x) (1 + \alpha^2) \kappa_1(1)
    \end{equation*}
    Noting that $\kappa_1(1)=1$ and assuming the induction holds for $K_{\ell-1}$, then
    \begin{equation*}
        K_{\ell}(\x,\x) = K_{\ell-1}(\x,\x)(1+\alpha^2) = \norm{\x}^2(1+\alpha^2)^{\ell-1}(1+\alpha^2) = \norm{\x}^2(1+\alpha^2)^{\ell}.
    \end{equation*}
\end{proof}

\begin{corollary} \label{cor:norm factor}
    For inputs in $\Spdm$, $K_{\ell}(\x,\x) = (1+\alpha^2)^{\ell}$.
\end{corollary} 

\vspace{0.2cm}
\subsection{Notation: ResNTK  in  $\Spdm$}
\label{Sec:ResNTK_model}

We next assume that $\x,\z \in \Spdm$ and let $u=\x^T\z$. Then, using the corollary above, ResNTK can be expressed as follows
\begin{eqnarray} \label{eq:app ResNTK}
    \resntkl{L}(u) &=& \frac{1}{2L(1+\alpha^2)^{L-1}} \sum_{\ell=1}^L B_{\ell+1}(u) \left[ (1+\alpha^2)^{\ell-1} \kappa_1 \left( \frac{K_{\ell-1}(u)}{(1+\alpha^2)^{\ell-1}} \right) + K_{\ell-1}(u)\kappa_0 \left( \frac{K_{\ell-1}(u)}{(1+\alpha^2)^{\ell-1}} \right) \right]
\end{eqnarray}
where $K_{0}(u) = u$, $B_{L+1}(u) = 1$, and
\begin{eqnarray}
    \label{eq:Kl}
    K_{\ell}(u) &=& K_{\ell-1}(u) + \alpha^2 (1-\alpha^2)^{l-1} \kappa_1 \left( \frac{K_{\ell-1}(u)}{(1+\alpha^2)^{\ell-1}} \right), ~~~\ell=1,...,L-1 \\
    \label{eq:B}
    B_{\ell}(u) &=& B_{\ell+1}(u) \left[1+\alpha^2\kappa_0\left( \frac{K_{\ell-1}(u)}{(1+\alpha^2)^{\ell-1}} \right) \right], ~~~ \ell=L,\ldots,2
\end{eqnarray}
and $\kappa_0$ and $\kappa_1$ are defined as
\begin{align}
    \kappa_0(u) &= \frac{1}{\pi}(\pi-acos(u))
    \label{eq:Kappa0 supp}\\ 
    \kappa_1(u) &= \frac{1}{\pi}\left(u \cdot (\pi-acos(u)) + \sqrt{1-u^2}\right).
    \label{eq:Kappa1 supp}
\end{align}

We further define the following for the expansion near -1 (small $t>0$): 
\begin{eqnarray}
    \label{eq:nu_ell}
    \nu_{\ell} &=& \frac{K_{\ell-1}(-1+t)}{(1+\alpha^2)^{\ell-1}} \\
    \label{eq:beta_ell}
    \beta_{\ell} &=& \kappa_1 \left( \nu_{\ell} \right) \\
    \label{eq:eta_ell}
    \eta_{\ell} &=& \kappa_0 \left( \nu_{\ell} \right)
\end{eqnarray}
for $\ell=1,2,...$, and $\beta_0 = \eta_L = 0$.
Note that $\beta_{\ell}, \eta_{\ell} \in [0,1]$ due to the image of the arc-cosine kernels.

\subsection{Expansion near 1}
\begin{lemma}\cite{bietti2020deep} \label{lem:arccosine around 1}
    The arc-cosine kernels near 1 satisfy
    \begin{eqnarray}
    \label{eq:k0 around 1}
        \kappa_0(1-t) &=& 1-\cpie t^{1/2}+\mathcal{O}(t^{3/2}) \\
    \label{eq:k1 around 1}
        \kappa_1(1-t) &=& 1-t+ \frac{2\sqrt{2}}{3\pi}t^{3/2}+\mathcal{O}(t^{5/2}).
    \end{eqnarray}
\end{lemma}

\begin{lemma} \label{lem: K_L around 1}
    For small $t>0$, $K_{\ell}(1-t) = (1+\alpha^2)^{\ell}(1-t)+o(t)$, where $K_{\ell}$ is defined in \eqref{eq:Kl}.
\end{lemma}

\begin{proof}
    We prove this by induction. For $\ell=0$,  $K_{0}(1-t) = 1-t$, 
    trivially satisfying the lemma. Suppose the lemma holds for $K_{\ell-1}(1-t)$, using \eqref{eq:Kl}
    \begin{eqnarray*}
        K_{\ell}(1-t) &=& K_{\ell-1}(1-t)+\alpha^2(1+\alpha^2)^{\ell-1} \kappa_1\left(\frac{K_{\ell-1}(1-t)}{(1+\alpha^2)^{\ell-1}}\right) \\ 
        &=& (1+\alpha^2)^{\ell-1}(1-t) + o(t) + \alpha^2(1+\alpha^2)^{\ell-1} \kappa_1\left(\frac{(1+\alpha^2)^{\ell-1}(1-t) + o(t)}{(1+\alpha^2)^{\ell-1}}\right) \\
        &=& (1+\alpha^2)^{\ell-1}(1-t) + o(t) + \alpha^2(1+\alpha^2)^{\ell-1} \kappa_1(1-t+o(t)) \\
        &=& (1+\alpha^2)^{\ell-1}(1-t) + \alpha^2(1+\alpha^2)^{\ell-1} (1-t) + o(t) = (1+\alpha^2)^{\ell}(1-t)+o(t),
    \end{eqnarray*}
    where the leftmost equality in the last line is due to \eqref{eq:k1 around 1}.
\end{proof}

\begin{lemma} \label{lem: kappas around 1}
    With small $t>0$,
\begin{eqnarray*}
    \kappa_0\left(\frac{K_{\ell-1}(1-t)}{(1+\alpha^2)^{\ell-1}}\right) &=& 1-\cpie t^{1/2}+o(t) \\
    \kappa_1\left(\frac{K_{\ell-1}(1-t)}{(1+\alpha^2)^{\ell-1}}\right) &=& 1 - t + o(t).
\end{eqnarray*}
\end{lemma}

\begin{proof}
Using Lemma \ref{lem: K_L around 1}, for small $t>0$,
\begin{equation*}
    \frac{K_{\ell-1}(1-t)}{(1+\alpha^2)^{\ell-1}} = \frac{(1+\alpha^2)^{\ell-1}(1-t)+o(t)}{(1+\alpha^2)^{\ell-1}} = 1-t+o(t).
\end{equation*}
Next, using \eqref{eq:k0 around 1}
\begin{equation*}
    \kappa_0\left(\frac{K_{\ell-1}(1-t)}{(1+\alpha^2)^{\ell-1}}\right) = \kappa_0(1-t+o(t)) = 1-\cpie t^{1/2}+o(t),
\end{equation*}
and using \eqref{eq:k1 around 1}
\begin{equation*}
    \kappa_1\left(\frac{K_{\ell-1}(1-t)}{(1+\alpha^2)^{\ell-1}}\right) = \kappa_1(1-t+o(t)) = 1 - t + o(t).
\end{equation*}
\end{proof}

\begin{lemma} \label{lem:B_L around 1}
    With small $t>0$,
    \begin{equation*}
        B_{\ell+1}(1-t) = (1+\alpha^2)^{L-\ell} - \frac{\sqrt{2}\,\alpha^2}{\pi} (1+\alpha^2)^{L-\ell-1} (L-\ell) \, t^{1/2} + \mathcal{O}(t),
    \end{equation*}
    where $B_{\ell}$ is defined in \eqref{eq:B}.
\end{lemma}
\begin{proof}
With small $t>0$, we use Lemma \ref{lem: kappas around 1} to simplify \eqref{eq:B} as follows:
\begin{equation*}
    B_{\ell}(1-t) = B_{\ell+1}(1-t) \left[1+\alpha^2 \left( 1-\cpie t^{1/2}+o(t) \right) \right].
\end{equation*}
Since $B_{L+1}=1$, resolving the recursion yields
    \begin{equation*} \label{eq:B around 1}
        B_{\ell+1}(1-t) = \left( 1 + \alpha^2 - \frac{\sqrt{2}\,\alpha^2}{\pi} t^{1/2} + \mathcal{O}(t^{3/2}) \right)^{L-\ell}.
\end{equation*}
This can be simplified as follows
\begin{equation*}
    B_{\ell+1}(1-t) = \sum_{i=0}^{L-\ell} {{L-\ell} \choose i} \left(1 + \alpha^2 + \mathcal{O}(t^{3/2}) \right)^{L-\ell-i} \left(- \frac{\sqrt{2}\,\alpha^2}{\pi} t^{1/2}  + \mathcal{O}(t^{3/2}) \right)^i.
\end{equation*}
Grouping together all $\mathcal{O}(t)$ terms, we finally obtain
    \begin{equation*}
        B_{\ell+1}(1-t) = (1+\alpha^2)^{L-\ell} - \frac{\sqrt{2}\,\alpha^2}{\pi} (1+\alpha^2)^{L-\ell-1} (L-\ell) \, t^{1/2} + \mathcal{O}(t).
    \end{equation*}
\end{proof}

We next prove Lemma 4.5 from the paper.
\begin{lemma} \label{lem:c1}
    For inputs in $\Spdm$ and near +1, if $\alpha > 0$ and $L \geq 1$
    \begin{equation*}
        \resntkl{L}(1-t) = 1 + c_1 t^{1/2} +o(t^{1/2})
    \end{equation*}
    where 
    \begin{equation*}
        c_1 = -\frac{1+\alpha^2L}{\sqrt{2}\pi(1+\alpha^2)}.
    \end{equation*}
\end{lemma}

\begin{proof}
Rewrite \eqref{eq:app ResNTK} as $\resntkl{L}(1-t) = C \sum_{\ell=1}^{L}X_{\ell}Y_{\ell}$, where:
\begin{eqnarray*}
    C &=& \frac{1}{2L(1+\alpha^2)^{L-1}}\\
    X_{\ell} &=& (1+\alpha^2)^{\ell-1} \kappa_1\left(\frac{K_{\ell-1}(1-t)}{(1+\alpha^2)^{\ell-1}}\right) + K_{\ell-1}(1-t)\kappa_0\left(\frac{K_{\ell-1}(1-t)}{(1+\alpha^2)^{\ell-1}}\right) \\
    Y_{\ell} &=& B_{\ell+1}(1-t).
\end{eqnarray*}
Using Lemmas~\ref{lem: K_L around 1} and~\ref{lem: kappas around 1}, for small $t>0$,
\begin{eqnarray*}
    X_{\ell} &=& (1+\alpha^2)^{\ell-1}(1 - t + o(t))+((1+\alpha^2)^{\ell-1}(1-t)+o(t)) \left(1-\cpie t^{1/2}+o(t)\right) \\
    &=& (1+\alpha^2)^{\ell-1}(1 - t)+(1+\alpha^2)^{\ell-1}(1-t) \left(1-\cpie t^{1/2}\right) +\mathcal{O}(t) \\
    &=& (1+\alpha^2)^{\ell-1}(1-t) \left(2-\cpie t^{1/2}\right) +\mathcal{O}(t) = (1+\alpha^2)^{\ell-1} \left(2-\cpie t^{1/2} \right) + o(t^{1/2}).
\end{eqnarray*}
Using Lemma~\ref{lem:B_L around 1} each term in the sum can be written as
\begin{eqnarray*}
    X_{\ell}Y_{\ell} &=&  \left[(1+\alpha^2)^{\ell-1} \left(2-\cpie t^{1/2}\right) \right] \left[(1+\alpha^2)^{L-\ell} - \frac{\alpha^2\sqrt{2}}{\pi}(1+\alpha^2)^{L-\ell-1} (L-\ell) \, t^{1/2}\right] + \mathcal{O}(t) \\
    &=& \left[2(1 + \alpha^2)^{L-1} - \cpie \left( 2\alpha^2 (1 + \alpha^2)^{L - 2}(L - \ell) + (1 + \alpha^2)^{L - 1} \right) t^{1/2}\right] + \mathcal{O}(t) \\
    &=& (1+\alpha^2)^{L-1} \left[2 - \cpie \left(\frac{2\alpha^2(L - \ell)}{1 + \alpha^2} + 1\right) t^{1/2}\right] + \mathcal{O}(t)
\end{eqnarray*}
Recall that $C = \frac{1}{2 L (1+\alpha^2)^{L-1}}$
\begin{equation*}
    C X_{\ell}Y_{\ell} = \frac{1}{2L} \left[2 - \cpie \left(\frac{2\alpha^2(L - \ell)}{1 + \alpha^2} + 1\right) t^{1/2}\right] + \mathcal{O}(t)
\end{equation*}
Summing over the layers
\begin{equation*}
    \resntkl{L}(1-t) = C\sum_{\ell=1}^L X_{\ell}Y_{\ell} = 1 - \frac{1}{\sqrt{2}\pi L} \left[ \frac{\alpha^2 L(L-1)}{1+\alpha^2} + L \right] t^{1/2}+\mathcal{O}(t) = 1 - \frac{1+\alpha^2L}{\sqrt{2}\pi(1+\alpha^2)} t^{1/2}+o(t^{1/2}).
\end{equation*}
\end{proof}

\subsection{Expansion near -1}

Here we investigate the expansion of ResNTK near -1. We consider two cases. First, with $\alpha>0$ such that $\alpha^2 L$ does not vanish as $L$ grows, and secondly, with $\alpha>0$ and $\alpha^2L \ll 1$.

\subsubsection{$\alpha>0$ such that $\alpha^2L \not\ll 1$}

\begin{lemma} \label{lem:arccosine around m1 app}
\cite{bietti2020deep}
    The arc-cosine kernels near -1 satisfy
\begin{equation} \label{eq:k0 around m1}
        \kappa_0(-1+t) = \cpie t^{1/2}+\mathcal{O}(t^{3/2})
    \end{equation}
    \begin{equation} \label{eq:k1 around m1}
        \kappa_1(-1+t) = \frac{2\sqrt{2}}{3\pi} t^{3/2} + \mathcal{O}(t^{5/2}).
    \end{equation}
\end{lemma}

\begin{lemma} \label{lem:K_L around m1}
    With small $t>0$,
    \begin{equation*}
        K_{\ell}(-1+t) = -1+t+\alpha^2\sum_{j=0}^{\ell}(1+\alpha^2)^{j-1}\beta_{j}+\mathcal{O}(t^{3/2}),
    \end{equation*}
    where $\beta_{\ell}$ as defined in \eqref{eq:beta_ell}.
\end{lemma}
\begin{proof}
    With $\ell=0$, $K_{0}(-1+t) = -1+t$, trivially satisfying the lemma. Suppose the lemma holds for $K_{\ell-1}(-1+t)$. Then, using \eqref{eq:Kl} and \eqref{eq:beta_ell}
    \begin{align*}
        K_{\ell}(-1+t) &= K_{\ell-1}(-1+t)+\alpha^2(1+\alpha^2)^{\ell-1}\kappa_1\left(\frac{K_{\ell-1}(-1+t)}{(1+\alpha^2)^{\ell-1}}\right)\\
        &= K_{\ell-1}(-1+t)+\alpha^2(1+\alpha^2)^{\ell-1}\beta_\ell.
    \end{align*}
    By the induction assumption
    \begin{align*}
        K_{\ell}(-1+t) &= -1+t+\alpha^2\sum_{j=0}^{\ell-1}(1+\alpha^2)^{j-1}\beta_j + \alpha^2(1+\alpha^2)^{\ell-1}\beta_{\ell}  + \mathcal{O}(t^{3/2})\\
        &= -1+t+\alpha^2\sum_{j=0}^{\ell}(1+\alpha^2)^{j-1}\beta_j + \mathcal{O}(t^{3/2}).
    \end{align*}
\end{proof}

The next Lemma ensures that $\beta_\ell$ is well defined (since $\kappa_1$ takes input in $[-1,1]$).
\begin{lemma} \label{lem:legal input}
    Let $\nu_\ell$ as defined in \eqref{eq:nu_ell}.
    Then, $\forall \ell \ge 1, ~~ \left| \nu_\ell \right| \le 1$.
\end{lemma}

\begin{proof}
    Using \eqref{eq:nu_ell} and Lemma \ref{lem:K_L around m1} we have
    \begin{align}
    \label{eq:nu_ell2}
        \nu_\ell &= \frac{-1+t+\alpha^2\sum_{j=0}^{\ell-1}(1+\alpha^2)^{j-1}\beta_j}{(1+\alpha^2)^{\ell-1}}
    \end{align}
    Since $\beta_0=0$, with $\ell=1$ $|\nu_1|=|-1+t| \le 1$. With $\ell>1$ using triangle inequality,
    \begin{align*}
        &|\nu_\ell| \le \left|\frac{-1+t+\alpha^2\sum_{j=0}^{\ell-2}(1+\alpha^2)^{j-1}\beta_j}{(1+\alpha^2)^{\ell-1}}\right| + \left|\frac{\alpha^2(1+\alpha^2)^{\ell-2}\beta_{\ell-1}}{(1+\alpha^2)^{\ell-1}}\right|.
    \end{align*}
    Noting that the first term is $\left| \frac{\nu_{\ell-1}}{1+\alpha^2}\right|$, and assuming by induction that the lemma is satisfied for $\nu_{\ell-1}$, then
    \begin{align*}
    &|\nu_\ell| \le \frac{1}{1+\alpha^2} + \frac{\alpha^2\beta_{\ell-1}}{1+\alpha^2} \le \frac{1}{1+\alpha^2} + \frac{\alpha^2}{1+\alpha^2} = 1,
    \end{align*}
    where the rightmost inequality is because by definition $\beta_{\ell} \in [0,1]$.
\end{proof}

\begin{lemma} \label{lem: develop point changes}
    Let $\delta_\ell = \frac{-1+\alpha^2\sum_{j=0}^{\ell-1}(1+\alpha^2)^{j-1}\beta_j}{(1+\alpha^2)^{\ell-1}}$. Then, $\forall \ell \ge 2, ~~ \left| \delta_\ell \right| < 1$.
\end{lemma}

\begin{proof}
    For $\ell=2$ we have $\left|\delta_2 \right| = \left| \frac{-1+\alpha^2\beta_1}{1+\alpha^2}\right| \le \max\{\frac{1}{1+\alpha^2}, \frac{\alpha^2-1}{1+\alpha^2} \} < 1$. Assume the lemma holds for $\ell-1$. We prove for $\ell$:
    \begin{align*}
        & \left | \delta_{\ell} \right | = \left | \frac{-1+\alpha^2\sum_{j=0}^{\ell-1}(1+\alpha^2)^{j-1}\beta_j}{(1+\alpha^2)^{\ell-1}} \right | = \left | \frac{-1+\alpha^2\sum_{j=0}^{\ell-2}(1+\alpha^2)^{j-1}\beta_j+\alpha^2(1+\alpha^2)^{\ell-1}\beta_{\ell}}{(1+\alpha^2)^{\ell-2}(1+\alpha^2)} \right| = \\
        & \left | \frac{\delta_{\ell-1}}{(1+\alpha^2)}+\frac{\alpha^2(1+\alpha^2)^{\ell-2}\beta_{\ell}}{(1+\alpha^2)^{\ell-2}(1+\alpha^2)} \right| = \left | \frac{\delta_{\ell-1}}{(1+\alpha^2)}+\frac{\alpha^2\beta_{\ell}}{(1+\alpha^2)} \right| \le^1 \left | \frac{\delta_{\ell-1}}{(1+\alpha^2)} \right | + \left |\frac{\alpha^2\beta_{\ell}}{(1+\alpha^2)} \right| <^2 \\
        & \frac{1}{(1+\alpha^2)} + \frac{\alpha^2}{(1+\alpha^2)} = 1,
    \end{align*}
    where $\le^1$ uses the triangle inequality, and $<^2$ is due to the induction hypothesis and the fact that $\forall \ell, \beta_{\ell} \in [0,1]$.
\end{proof}


\begin{lemma} \label{lem: beta near m1}
    With small $t>0$, $\forall \ell \in [L-1]$
    \begin{align*}
        \beta_{\ell} = \kappa_1 \left( \frac{-1+\alpha^2\sum_{j=0}^{\ell-1}(1+\alpha^2)^{j-1}\beta_j}{(1+\alpha^2)^{\ell-1}} \right) + \mathcal{O}(t).
    \end{align*}
\end{lemma}
\begin{proof}
    First, note that for $\ell=1$ we get this directly from Lemma~\ref{eq:k1 around m1}. For $\ell \ge 2$, using Lemma~\ref{lem:K_L around m1} and the definition in \eqref{eq:beta_ell}:
    \begin{align*}
        \beta_{\ell} = \kappa_1 \left( \frac{-1+t+\alpha^2\sum_{j=0}^{\ell-1}(1+\alpha^2)^{j-1}\beta_j}{(1+\alpha^2)^{\ell-1}} \right) = \kappa_1 \left( \frac{-1+\alpha^2\sum_{j=0}^{\ell-1}(1+\alpha^2)^{j-1}\beta_j}{(1+\alpha^2)^{\ell-1}} + \mathcal{O}(t) \right) = \kappa_1 \left( \delta_{\ell} + \mathcal{O}(t) \right),
    \end{align*}
    where $\delta_{\ell}$ is defined in Lemma~\ref{lem: develop point changes}. Note that from this lemma, $-1 < \delta_{\ell} < 1$. In this domain, $\kappa_1$ is infinitely differentiable, hence we get:
    \begin{align*}
        \beta_{\ell} = \kappa_1 \left( \delta_{\ell} \right) + \mathcal{O}(t) = \kappa_1 \left( \frac{-1+\alpha^2\sum_{j=0}^{\ell-1}(1+\alpha^2)^{j-1}\beta_j}{(1+\alpha^2)^{\ell-1}} \right) + \mathcal{O}(t).
    \end{align*}
\end{proof}

\begin{lemma} \label{lem: around m1 beta const}
    With small $t>0$, $\forall \ell \in [L-1]$
    \begin{align*}
        \beta_{\ell} = \tilde{c_{\ell}} + \mathcal{O}(t),
    \end{align*}
    where $\tilde{c_{\ell}} \in [0,1]$ does not depend on $t$.
\end{lemma}
\begin{proof}
    The proof is by induction. For $\ell=1$ we have from Lemma~\ref{lem: beta near m1}
    \begin{align*}
        \beta_{1} = \kappa_1 \left( \frac{-1}{(1+\alpha^2)} \right) + \mathcal{O}(t) = \tilde{c_1} + \mathcal{O}(t).
    \end{align*}
    Suppose the lemma holds for $\beta_{\ell-1}$ and show for $\beta_{\ell}$
    \begin{align*}
        & \beta_{\ell} = \kappa_1 \left( \frac{-1+\alpha^2\sum_{j=0}^{\ell-1}(1+\alpha^2)^{j-1}\beta_j}{(1+\alpha^2)^{\ell-1}} + \mathcal{O}(t) \right) = \kappa_1 \left( \frac{-1+\alpha^2\sum_{j=0}^{\ell-1}(1+\alpha^2)^{j-1}(\tilde{c_j}+\mathcal{O}(t))}{(1+\alpha^2)^{\ell-1}} + \mathcal{O}(t) \right) = \\
        & \kappa_1 \left( \frac{-1+\alpha^2\sum_{j=0}^{\ell-1}(1+\alpha^2)^{j-1}\tilde{c_j}}{(1+\alpha^2)^{\ell-1}} + \mathcal{O}(t) \right) = \kappa_1 \left( \frac{-1+\alpha^2\sum_{j=0}^{\ell-1}(1+\alpha^2)^{j-1}\tilde{c_j}}{(1+\alpha^2)^{\ell-1}} \right) + \mathcal{O}(t) = \tilde{c_{\ell}} + \mathcal{O}(t),
    \end{align*}
    where the leftmost equality in the second line is from Lemma~\ref{lem: beta near m1}. The definition of $\tilde{c_{\ell}}$ directly implies that $\tilde{c_{\ell}} \in [0,1]$.
\end{proof}

\begin{lemma} \label{lem: eta near m1}
    With small $t>0$, and for $\ell=1$,
    \begin{align*}
        \eta_{1} = \cpie t^{1/2}+\mathcal{O}(t^{3/2}).
    \end{align*}
    For $\ell \ge 2$,
    \begin{align*}
        \eta_{\ell} = \kappa_0 \left( \frac{-1+\alpha^2\sum_{j=0}^{\ell-1}(1+\alpha^2)^{j-1}\beta_j}{(1+\alpha^2)^{\ell-1}} \right) + \mathcal{O}(t).
    \end{align*}
where $\eta_{\ell}$ is defined in \eqref{eq:eta_ell}.
\end{lemma}

\begin{proof}
    First, note that for $\ell=1$ we get this directly from Lemma~\ref{eq:k0 around m1}. For $\ell \ge 2$, using Lemma~\ref{lem:K_L around m1} and the definition \eqref{eq:eta_ell}:
    \begin{align*}
        \eta_{\ell} = \kappa_0 \left( \frac{-1+t+\alpha^2\sum_{j=0}^{\ell-1}(1+\alpha^2)^{j-1}\beta_j}{(1+\alpha^2)^{\ell-1}} \right) = \kappa_0 \left( \frac{-1+\alpha^2\sum_{j=0}^{\ell-1}(1+\alpha^2)^{j-1}\beta_j}{(1+\alpha^2)^{\ell-1}} + \mathcal{O}(t) \right) = \kappa_0 \left( \delta_{\ell} + \mathcal{O}(t) \right) .
    \end{align*}
    where $\delta_{\ell}$ is defined in Lemma~\ref{lem: develop point changes}. Note that from this lemma, $-1 < \delta_{\ell} < 1$. In this domain, $\kappa_0$ is infinitely differentiable, hence we get:
    \begin{align*}
        \eta_{\ell} = \kappa_0 \left( \delta_{\ell} \right) + \mathcal{O}(t) = \kappa_0 \left( \frac{-1+\alpha^2\sum_{j=0}^{\ell-1}(1+\alpha^2)^{j-1}\beta_j}{(1+\alpha^2)^{\ell-1}} \right) + \mathcal{O}(t)
    \end{align*}
\end{proof}

\begin{lemma} \label{lem: around m1 eta const}
    With small $t>0$, $\forall \ell \ge 2$
    \begin{align*}
        \eta_{\ell} = \tilde{d_{\ell}} + \mathcal{O}(t),
    \end{align*}
    where $\tilde{d_{\ell}} \in [0,1]$ does not depend on $t$.
\end{lemma}
\begin{proof}
    The proof is by induction. For $\ell=2$ we have from Lemma~\ref{lem: eta near m1}
    \begin{align*}
        \eta_{2} = \kappa_0 \left( \frac{-1}{(1+\alpha^2)} \right) + \mathcal{O}(t) = \tilde{d_2} + \mathcal{O}(t).
    \end{align*}
    Suppose the lemma holds for $\eta_{\ell-1}$ and show for $\eta_{\ell}$
    \begin{align*}
        & \eta_{\ell} = \kappa_0 \left( \frac{-1+\alpha^2\sum_{j=0}^{\ell-1}(1+\alpha^2)^{j-1}\beta_j}{(1+\alpha^2)^{\ell-1}} + \mathcal{O}(t) \right) = \kappa_0 \left( \frac{-1+\alpha^2\sum_{j=0}^{\ell-1}(1+\alpha^2)^{j-1}(\tilde{c_j}+\mathcal{O}(t))}{(1+\alpha^2)^{\ell-1}} + \mathcal{O}(t) \right) = \\
        & \kappa_0 \left( \frac{-1+\alpha^2\sum_{j=0}^{\ell-1}(1+\alpha^2)^{j-1}\tilde{c_j}}{(1+\alpha^2)^{\ell-1}} + \mathcal{O}(t) \right) = \kappa_0 \left( \frac{-1+\alpha^2\sum_{j=0}^{\ell-1}(1+\alpha^2)^{j-1}\tilde{c_j}}{(1+\alpha^2)^{\ell-1}} \right) + \mathcal{O}(t) = \tilde{d_{\ell}} + \mathcal{O}(t),
    \end{align*}
    where the leftmost equality in the second line is from Lemma~\ref{lem: eta near m1}. The definition of $\tilde{d_{\ell}}$ directly implies that $\tilde{d_{\ell}} \in [0,1]$.
\end{proof}

\begin{lemma} \label{lem:B_L around m1}
    With small $t>0$,
    \begin{equation*}
        B_{\ell+1}(-1+t) = \displaystyle\prod_{i=\ell+1}^{L} (1+\alpha^2\eta_{i})
    \end{equation*}
where $\eta_{\ell}$ is defined in \eqref{eq:eta_ell}.
\end{lemma}
\begin{proof}
    Since $B_{L+1}=1$ and using \eqref{eq:B}
    \begin{equation*}
        B_{\ell+1}(-1+t) = \prod_{i=\ell+1}^{L} \left[ 1+\alpha^2\kappa_0\left( \frac{K_{i-1}(-1+t)}{(1+\alpha^2)^{i-1}} \right) \right] = \prod_{i=\ell+1}^{L} \left[ 1+\alpha^2 \eta_{\ell} \right]
    \end{equation*}
\end{proof}

We next prove Lemma 4.6 from the paper.
\begin{lemma} \label{lem:AppResNTK around m1}
    For inputs in $\Spdm$ and near -1, if $\alpha > 0$ and $L \ge 2$ then
    \begin{equation*}
        \resntkl{L}(-1+t) = 
        p_{-1}(t)+c_{-1} t^{1/2} + o(t^{1/2}),
    \end{equation*}
    with
    \begin{align*}
        |c_{-1}| \leq \frac{1}{\sqrt{2}\pi(1+\alpha^2)L}.
    \end{align*}
\end{lemma}
\begin{proof}
    Rewrite \eqref{eq:app ResNTK} as $\resntkl{L}(-1+t) = C \sum_{\ell=1}^{L}X_{\ell}Y_{\ell}$, where:
    \begin{eqnarray*}
    C &=& \frac{1}{2L(1+\alpha^2)^{L-1}}\\
    X_{\ell} &=& (1+\alpha^2)^{\ell-1} \kappa_1\left(\frac{K_{\ell-1}(-1+t)}{(1+\alpha^2)^{\ell-1}}\right) + K_{\ell-1}(-1+t)\kappa_0\left(\frac{K_{\ell-1}(-1+t)}{(1+\alpha^2)^{\ell-1}}\right) = (1+\alpha^2)^{\ell-1} \beta_{\ell} + K_{\ell-1}(-1+t)\eta_{\ell}\\
    Y_{\ell} &=& B_{\ell+1}(-1+t).
\end{eqnarray*}
By plugging Lemma~\ref{lem:K_L around m1} into the definition of $X_\ell$ we have 
\begin{eqnarray*}
    X_{\ell} &=& (1+\alpha^2)^{\ell-1}\beta_{\ell}+\left(-1+\alpha^2\sum_{j=0}^{\ell-1}(1+\alpha^2)^{j-1}\beta_{j}\right) \eta_{\ell}+\mathcal{O}(t).
\end{eqnarray*}
Using Lemma~\ref{lem:B_L around m1} the sum can be written as
\begin{align*}
    \sum_{\ell=1}^L X_{\ell}Y_{\ell} = \sum_{\ell=1}^L \left((1+\alpha^2)^{\ell-1}\beta_{\ell}+\left(-1+\alpha^2\sum_{j=0}^{\ell-1}(1+\alpha^2)^{j-1}\beta_{j}\right) \eta_{\ell}\right) \displaystyle\prod_{i=\ell+1}^{L} (1+\alpha^2\eta_{i})+\mathcal{O}(t).
\end{align*}
From Lemma~\ref{lem: eta near m1}, there is a difference between $\ell=1$ and $\ell \ge 2$. For $\ell=1$: 
\begin{align*}
   & X_{1}Y_{1} = \left((1+\alpha^2)^{0}\beta_{1} + \left(-1+\alpha^2\sum_{j=0}^{0}(1+\alpha^2)^{j-1}\beta_{j}\right) \eta_{1}\right) \displaystyle\prod_{i=1+1}^{L} (1+\alpha^2\eta_{i})+\mathcal{O}(t) = \\
   & -\eta_{1}\displaystyle\prod_{i=2}^{L} (1+\alpha^2\eta_{i})+\mathcal{O}(t) = -\left(\displaystyle\prod_{i=2}^{L} (1+\alpha^2\eta_{i})\right) \cpie t^{1/2}+\mathcal{O}(t)
\end{align*}
Using Lemma~\ref{lem: around m1 eta const} this simplifies to
\begin{align*}
    X_{1}Y_{1} = -\left(\displaystyle\prod_{i=2}^{L} (1+\alpha^2(\tilde{d_{i}}+\mathcal{O}(t)))\right) \cpie t^{1/2}+\mathcal{O}(t) = -\left(\displaystyle\prod_{i=2}^{L} (1+\alpha^2\tilde{d_{i}})\right) \cpie t^{1/2}+\mathcal{O}(t)
\end{align*}

For $\ell \geq 2$, using Lemmas~\ref{lem: around m1 beta const},~\ref{lem: around m1 eta const}
\begin{align*}
   & X_{\ell}Y_{\ell} = \left((1+\alpha^2)^{\ell-1}\beta_{\ell} + \left(-1+\alpha^2\sum_{j=0}^{\ell-1}(1+\alpha^2)^{j-1}\beta_{j}\right) \eta_{\ell}\right) \displaystyle\prod_{i=\ell+1}^{L} (1+\alpha^2\eta_{i})+\mathcal{O}(t) = \\
   & \left((1+\alpha^2)^{\ell-1}(\tilde{c_{\ell}}+\mathcal{O}(t)) + \left(-1+\alpha^2\sum_{j=0}^{\ell-1}(1+\alpha^2)^{j-1}(\tilde{c_{j}}+\mathcal{O}(t))\right) (\tilde{d_{\ell}}+\mathcal{O}(t))\right) \displaystyle\prod_{i=\ell+1}^{L} (1+\alpha^2(\tilde{d_{i}}+\mathcal{O}(t)))+\mathcal{O}(t) = \\
   & \left((1+\alpha^2)^{\ell-1}\tilde{c_{\ell}} + \left(-1+\alpha^2\sum_{j=0}^{\ell-1}(1+\alpha^2)^{j-1}\tilde{c_{j}}\right) \tilde{d_{\ell}}\right) \displaystyle\prod_{i=\ell+1}^{L} (1+\alpha^2\tilde{d_{i}})+\mathcal{O}(t)
\end{align*}

The sum can be rewritten as
\begin{align*}
     & \sum_{\ell=1}^L X_{\ell}Y_{\ell} = \left(\sum_{\ell=2}^L \left((1+\alpha^2)^{\ell-1}\tilde{c_{\ell}} + \left(-1+\alpha^2\sum_{j=0}^{\ell-1}(1+\alpha^2)^{j-1}\tilde{c_{j}}\right) \tilde{d_{\ell}}\right) \displaystyle\prod_{i=\ell+1}^{L} (1+\alpha^2\tilde{d_{i}}) \right) - \left(\displaystyle\prod_{i=2}^{L} (1+\alpha^2\tilde{d_{i}})\right) \cpie t^{1/2}+\mathcal{O}(t).
\end{align*}
Multiplying this by the normalization factor $C$ we have
\begin{align*}
    & \resntkl{L}(-1+t) = C\sum_{\ell=1}^L X_{\ell}Y_{\ell} = \frac{1}{2L(1+\alpha^2)^{L-1}}\sum_{\ell=1}^L X_{\ell}Y_{\ell} = p_{-1}(t) + c_{-1} t^{1/2} + o(t^{1/2}),
\end{align*}
where
\begin{align*}
    & p_{-1}(t) = \frac{1}{2L(1+\alpha^2)^{L-1}}\left(\sum_{\ell=2}^L \left((1+\alpha^2)^{\ell-1}\tilde{c_{\ell}} + \left(-1+\alpha^2\sum_{j=0}^{\ell-1}(1+\alpha^2)^{j-1}\tilde{c_{j}}\right) \tilde{d_{\ell}}\right) \displaystyle\prod_{i=\ell+1}^{L} (1+\alpha^2\tilde{d_{i}}) \right)\\
    & c_{-1} = -\frac{1}{2L(1+\alpha^2)^{L-1}}\left(\displaystyle\prod_{i=2}^{L} (1+\alpha^2\tilde{d_{i}})\right) \cpie.
\end{align*}
From Lemma~\ref{lem: around m1 eta const},
\begin{align*}
    \left| c_{-1} \right| = \left | \frac{1}{2L(1+\alpha^2)^{L-1}} \left(\displaystyle\prod_{i=2}^{L} (1+\alpha^2\tilde{d_{i}})\right) \cpie\right | =
    \left | \frac{1}{\sqrt{2}\pi L(1+\alpha^2)^{L-1}} \left(\displaystyle\prod_{i=2}^{L} (1+\alpha^2\tilde{d_{i}})\right) \right | \le \frac{\sqrt{2}(1+\alpha^2)^{L-2}}{2\pi L(1+\alpha^2)^{L-1}} = \frac{1}{\sqrt{2}\pi(1+\alpha^2)L}.
\end{align*}
\end{proof}

\subsubsection{Vanishing regime $\alpha^2L \ll 1$}

For the case where $\alpha^2 L\xrightarrow{} 0$ with $L \xrightarrow{} \infty$ (which implies $(1+\alpha^2)^j \approx 1, \forall j\in[L]$), the analysis takes the following form. The next Lemma is analogous to Lemma~\ref{lem:K_L around m1}.
\begin{lemma} \label{lem: K_L around m1 vanishing}
    With small $t>0$ and $\alpha^2L \ll 1$,
    \begin{equation*}
        K_{\ell}(-1+t) = -1+t+\mathcal{O}(t^{3/2}).
    \end{equation*}
\end{lemma}
\begin{proof}
    With $\ell=0$, $K_{0}(-1+t) = -1+t$, trivially satisfying the lemma. Suppose the lemma holds for $K_{\ell-1}(-1+t)$. Then, using \eqref{eq:Kl} and \eqref{eq:beta_ell}
    \begin{align*}
        K_{\ell}(-1+t) &= K_{\ell-1}(-1+t)+\alpha^2(1+\alpha^2)^{\ell-1}\kappa_1\left(\frac{K_{\ell-1}(-1+t)}{(1+\alpha^2)^{\ell-1}}\right)\\
        &= K_{\ell-1}(-1+t)+\alpha^2\kappa_1\left(K_{\ell-1}(-1+t)\right).
    \end{align*}
    Where the last equality is from $\alpha^2 \ll 1$. By the induction assumption
    \begin{align*}
        K_{\ell}(-1+t) &= (-1+t+\mathcal{O}(t^{3/2}))+\alpha^2\kappa_1\left(-1+t+\mathcal{O}(t^{3/2})\right) = -1+t+\mathcal{O}(t^{3/2}),
    \end{align*}
    where the last equality is directly from Lemma~\ref{lem:arccosine around m1 app}.
\end{proof}

The next Lemma is analogous to Lemma~\ref{lem:legal input}.
\begin{lemma} \label{lem:legal input stays -1}
    Let $\nu_\ell$ as defined in \eqref{eq:nu_ell}. Then, for $\alpha^2L \ll 1$, $\forall \ell \ge 1, ~~ \nu_\ell = -1 + \mathcal{O}(t)$.
\end{lemma}
\begin{proof}
    Using \eqref{eq:nu_ell2}, with $\ell=1$, $\nu_1=-1+t$. Assume the lemma is satisfied for $\nu_{\ell-1}$. Then, for $1 \le j \le \ell-1$,
    \begin{align*}
        \beta_{j} = \kappa_1(\nu_{j}) = \kappa_1(-1+\mathcal{O}(t)) = \mathcal{O}(t),
    \end{align*}
    where the rightmost equality is due to \eqref{eq:k1 around m1}. Therefore, using \eqref{eq:nu_ell2} and $(1+\alpha^2)^{\ell-1} \approx 1$ we obtain
    \begin{align*}
        & \nu_{\ell} = \frac{-1+t+\alpha^2\sum_{j=0}^{\ell-1}(1+\alpha^2)^{j-1}\beta_j}{(1+\alpha^2)^{\ell-1}} = -1+t+\alpha^2\sum_{j=0}^{\ell-1}\mathcal{O}(t) = -1 + \mathcal{O}(t).
    \end{align*}
\end{proof}
Combining this lemma with lemma~\ref{lem:arccosine around m1 app} we get the following lemmas (analogous to \ref{lem: beta near m1}, \ref{lem: eta near m1}):
\begin{lemma} \label{lem: beta near m1 vanishing}
    With $\alpha^2L \rightarrow 0$, $\forall \ell \in [L-1]$, $\beta_{\ell} = \kappa_1(\nu_{\ell}) = \kappa_1(-1+t) = \mathcal{O}(t)$.
\end{lemma}
\begin{lemma} \label{lem: eta near m1 vanishing}
    With $\alpha^2L \rightarrow 0$, $\forall \ell \in [L-1]$, $\eta_{\ell} = \kappa_0(\nu_{\ell}) = \kappa_0(-1+t) = \cpie t^{1/2} + \mathcal{O}(t)$.
\end{lemma}
\begin{lemma} \label{lem:B vanishing}
     With $\alpha^2L \rightarrow 0$, $\forall \ell \in [L-1]$, 
     \begin{align*}
         B_{\ell+1}(-1+t) &= 1 + (L-\ell)\frac{\sqrt{2}\alpha^2}{\pi} t^{1/2} + \mathcal{O}(t).
     \end{align*}
\end{lemma}
\begin{proof}
    Using lemma~\ref{lem:B_L around m1}, the expansion of B around -1 can be written in this regime as:
    \begin{align*}
        B_{\ell+1}(-1+t) &= \displaystyle\prod_{i=\ell+1}^{L} (1+\alpha^2\eta_{i}) = \displaystyle\prod_{i=\ell+1}^{L} \left(1+\frac{\sqrt{2}\alpha^2}{\pi} t^{1/2}\right)+\mathcal{O}(t) = \left(1+\frac{\sqrt{2}\alpha^2}{\pi} t^{1/2}\right)^{L-\ell} + \mathcal{O}(t)\\
        &= 1 + (L-\ell)\frac{\sqrt{2}\alpha^2}{\pi} t^{1/2} + \mathcal{O}(t).
    \end{align*}
\end{proof}

We next prove Lemma 4.9 from the paper.
\begin{lemma} \label{lem:AppResNTK around m1 second regime}
    For inputs in $\Spdm$ and near -1, if $\alpha^2 L \ll 1$ then
    \begin{equation*}
        \resntkl{L}(-1+t) = 
        c_{-1} t^{1/2} + o(t^{1/2})
    \end{equation*}
    with
    \begin{equation*}
        c_{-1} = -\frac{1}{\sqrt{2}\pi}
    \end{equation*}
\end{lemma}

\begin{proof}
    Rewrite \eqref{eq:app ResNTK} $\resntkl{L}(-1 + t) = C \sum_{\ell=1}^{L}X_{\ell}Y_{\ell}$, where:
    \begin{eqnarray*}
    C &=& \frac{1}{2L(1+\alpha^2)^{L-1}} \approx \frac{1}{2L}\\
    X_{\ell} &=& (1+\alpha^2)^{\ell-1} \kappa_1\left(\frac{K_{\ell-1}(-1+t)}{(1+\alpha^2)^{\ell-1}}\right) + K_{\ell-1}(-1+t)\kappa_0\left(\frac{K_{\ell-1}(-1+t)}{(1+\alpha^2)^{\ell-1}}\right) = (1+\alpha^2)^{\ell-1} \beta_{\ell} + K_{\ell-1}(-1+t)\eta_{\ell}\\
    Y_{\ell} &=& B_{\ell+1}(-1+t).
\end{eqnarray*}
Using $(1+\alpha^2) \approx 1$ and Lemmas~\ref{lem: K_L around m1 vanishing},~\ref{lem: beta near m1 vanishing} and \ref{lem: eta near m1 vanishing}
\begin{align*}
    X_{\ell} = (1+\alpha^2)^{\ell-1} \beta_{\ell} + K_{\ell-1}(-1+t)\eta_{\ell} = -\cpie t^{1/2} + \mathcal{O}(t).
\end{align*}

Using the above and Lemma \ref{lem:B vanishing}, we have
\begin{align*}
    X_{\ell}Y_{\ell} = \left(\left(-\cpie t^{1/2} + \mathcal{O}(t)\right)\left(1 + (L-\ell)\frac{\sqrt{2}\alpha^2}{\pi} t^{1/2} + \mathcal{O}(t)\right) \right) = -\cpie t^{1/2} + \mathcal{O}(t).
\end{align*}
Consequently,
\begin{align*}
    \resntkl{L}(-1+t) &= C\sum_{\ell=1}^L X_{\ell}Y_{\ell} = C\sum_{\ell=1}^L \left(-\cpie t^{1/2} + \mathcal{O}(t) \right) \\
    &= \frac{1}{2L} \left(-\frac{\sqrt{2}L}{\pi} t^{1/2}\right) + \mathcal{O}(t) = -\frac{1}{\sqrt{2} \pi} t^{1/2} + \mathcal{O}(t) = -\frac{1}{\sqrt{2} \pi} t^{1/2} + o(t^{1/2})
\end{align*}
\end{proof}

Note that with the conditions of $\alpha^2L \xrightarrow{} 0$ with $L \xrightarrow{} \infty$, using Lemma \ref{lem:c1},
\begin{equation*}
    c_1 = -\frac{1+\alpha^2L}{\sqrt{2}\pi(1+\alpha^2)} \xrightarrow[]{L \xrightarrow{} \infty} -\frac{1}{\sqrt{2}\pi}.
\end{equation*}
This is indeed the case when $\alpha=L^{-\gamma}$ with $0.5 < \gamma \le 1$. In this case we have from Lemma \ref{lem:AppResNTK around m1 second regime} that $c_1 = c_{-1}$, implying that the odd frequencies decay faster than $\mathcal{O}(k^{-d})$. If however $\alpha = L^{-1/2}$ then for all $L$, $\alpha^2 L=1$ and $c_1$ approaches $-\sqrt{2}/\pi$ and all the frequencies decay exactly at the rate of $\mathcal{O}(k^{-d})$.

\section{Steepness of FC-NTK}

\begin{lemma} \cite{bietti2020deep} \label{lem:app Laplace}
    With small $t>0$,
    \begin{align*}
        \lap(1-t) = e^{-c\sqrt{2t}} = 1-c\sqrt{2t}+\mathcal{O}(t),
    \end{align*}
    where $\lap$ is defined in equation (8) in the paper.
\end{lemma}

We next prove Lemma 5.2 from the paper.
\begin{lemma}
With small $t>0$,
\begin{equation*}
        \ntkl{L}(1-t) = 1 - \frac{L}{\pi\sqrt{2}}t^{1/2}+o(t^{1/2}).
    \end{equation*}
    Therefore, with $c = \frac{L}{2\pi}$, $\ntkl{L}(1-t)-\lap(1-t)=o(t^{1/2})$.
\end{lemma}

\begin{proof}
The proof is by induction on the unnormalized kernel $\tntkl{\ell}=(\ell+1)\ntkl{\ell}$.
With $\ell=1$:
\begin{align*}
    \tntkl{1}(1-t) &=  (1-t) \kappa_0(1-t)+\kappa_1(1-t) = (1-t)\left(1 - \cpie t^{1/2}+\mathcal{O}(t^{3/2})\right) + 1 + \mathcal{O}(t) \\
    &= 
    2 - \cpie t^{1/2} + o(t^{1/2}).
\end{align*}
Note that by the definition of $\tntkl{\ell}$
\begin{align*}
    \tntkl{\ell}(u) = \tntkl{\ell-1}(u)\kappa_0(\Sigma^{(\ell-1)}(u))+\Sigma^{(\ell)}(u).
\end{align*}
Using
\begin{align*}
    \Sigma^{(\ell)}(1-t)
    = 1 - t + o(t),
\end{align*}
that was proved in \cite{bietti2020deep}. Additionally, using the equation above and Lemma~\ref{lem:arccosine around m1 app}
\begin{align*}
    \kappa_0(\Sigma^{(\ell-1)}(1-t)) = \kappa_0(1 - t + o(t)) = 1-\cpie(t+o(t))^{1/2} + o(t^{1/2}) = 1 - \cpie t^{1/2}+o(t^{1/2}).
\end{align*}
Suppose the lemma holds for $j \le \ell-1$, then 
\begin{align*}
    \tntkl{\ell}(1-t) &= \tntkl{\ell-1}(1-t)\kappa_0(\Sigma^{(\ell-1)}(1-t)) + \Sigma^{(\ell)}(1-t) \\ 
    &= \ell \left(1 - \frac{\ell-1}{\pi\sqrt{2}}t^{1/2}+o(t^{1/2})\right) \left( 1 - \cpie t^{1/2}+o(t^{1/2}) \right) + 1 - t + o(t) \\
    &= 
    \ell + 1 - \frac{\ell(\ell+1)}{\pi\sqrt{2}} t^{1/2}+o(t^{1/2}).
\end{align*}
Using $\ntkl{L}=\frac{1}{L+1}\tntkl{L}$, the first part of the lemma is proven. Finally, using Lemma \ref{lem:app Laplace}, the relation to the Laplace kernel is immediate.
\end{proof}




\section{Proof of Theorem 4.8 from the paper}


\begin{theorem}
For ResNTK, as $L \rightarrow \infty$, with $\alpha = L ^{-\gamma}$, $0.5 < \gamma \leq 1$, for any two inputs $\x, \z \in \Sphere^{d-1}$, such that $1-|\x^T\z| \ge \delta > 0$ it holds that 
\begin{equation*}
    |\resntkl{L}(\x,\z)-\ntkl{1}(\x,\z)| = O(L^{1-2\gamma}).
\end{equation*}
\end{theorem}

\begin{proof}
We follow  the ResNTK notations in Sec. \ref{Sec:ResNTK_model}. We include an additional subscript $L$ to emphasize the dependence of $\alpha$ on $L$. Let
\begin{align*}
    u_{\ell, L} = \frac{K_{\ell,L}}{(1+\alpha^2)^{\ell}}, ~~ u_{0} = K_{0} = \x^T \z
\end{align*}
and assume that $-1+\delta < u_0 < 1 -\delta$.
Following these notations, and using Corollary \ref{cor:norm factor}, we obtain the following relation  
\begin{equation} \label{eq:U_relation}
    u_{\ell,L} = \frac{u_{\ell-1,L}+\alpha^2\kappa_1(u_{\ell-1,L})}{1+\alpha^2},
\end{equation}
which implies that 
\begin{equation}\label{eq:subsequent_u}
u_{\ell,L} - u_{\ell-1,L} = \frac{\alpha^2}{1+\alpha^2}(\kappa_1(u_{\ell-1,L})-u_{\ell-1,L}).
\end{equation}
We note that $\kappa_0, \kappa_1 : [-1,1] \rightarrow [0,1]$ and $\kappa_1'(s) = \kappa_0(s)$, and therefore, the derivative of 
the function $\kappa_1(s)-s$ is non-positive,  implying that $\kappa_1(s)-s$ is non-increasing. Therefore, the minimal value is attained at $s=1$ and the maximal value at $s=-1$. Since $\kappa_1(1)-1=0$ and $\kappa_1(-1)+1=1$ this means that $0 \leq \kappa_1(s)-s \leq 1$. Now, by the relation \eqref{eq:subsequent_u}, it is easy to see that $u_{\ell,L} \geq u_{\ell-1,L}$, which means that 
\begin{equation}\label{eq:u_increasing}
    u_0 \leq u_{1,L} \leq \ldots \leq u_{L-1,L}.
\end{equation}
In addition, we obtain the following upper bound for $u_{\ell,L}-u_0$ 
\begin{align*}
& u_{\ell,L}-u_0 = \sum_{i=1}^{\ell}(u_{i,L}-u_{i-1,L}) = \frac{\alpha^2}{1+\alpha^2}\sum_{i=1}^{\ell}(\kappa_1(u_{i-1,L})-u_{i-1,L}) \leq \frac{\alpha^2}{1+\alpha^2}(\kappa_1(u_0)-u_0)\ell, \end{align*} 
where the last inequality uses the observation  $u_0 \leq u_{i,L}$ and that $\kappa_1(s)-s$ is decreasing.
The last inequality is equivalent to 
\begin{equation} \label{eq:UpperBound_U}
    u_{\ell,L} \le u_{0} + \frac{\alpha^2}{1+\alpha^2}(\kappa_1(u_{0})-u_{0})\ell.
\end{equation}

For  $\alpha = L ^{-\gamma}$, we have $\frac{\alpha^2}{1+\alpha^2} = \frac{1}{1+L^{2\gamma}}$, and since $0 \le \kappa_1(s)-s \le 1$ this inequality implies that 
\begin{equation}
    u_{L-1,L} \le u_0 + \frac{L}{1+L^{2 \gamma}} \leq 1-\delta + L^{1-2\gamma}.
 \end{equation}
Therefore, for $\gamma > 0.5$ and $L$ sufficiently large, this yields a maximal bound $1-\delta'$ over the series \eqref{eq:u_increasing}, with $\delta > \delta'>0$.


Denote by
\begin{equation*}
    P_{\ell+1,L} = B_{\ell+1,L}(1+\alpha^2)^{-(L-\ell)} = \prod_{i=\ell}^{L-1} \frac{1+\alpha^2 \kappa_0(u_{i,L})}{1+\alpha^2},
\end{equation*}
and note that $P_{l+1,L} \in (0,1]$.
Since  $1-\frac{1+\alpha^2 \kappa_0(u_{i,L})}{1+\alpha^2}=\frac{\alpha^2(1-\kappa_0(u_{i,L}))}{1+\alpha^2}$ and for $a_k \in [0,1]$, $1 - \prod_{k=1}^{n}(1-a_k) \le \sum_{k=1}^{n} a_k$ (see Lemma \ref{eq:supporting_inequality}), we obtain 
\begin{align}
\label{eq:1mP}
    1 - P_{\ell+1,L} = 1 - \prod_{i=\ell}^{L-1}\left(1-\frac{\alpha^2(1-\kappa_0(u_{i,L}))}{1+\alpha^2}\right) \le \sum_{i=\ell}^{L-1} \frac{\alpha^2(1-\kappa_0(u_{i,L}))}{1+\alpha^2} = \frac{\alpha^2}{1+\alpha^2}\left(L-\ell-\sum_{i=\ell}^{L-1}\kappa_0(u_{i,L})\right).
\end{align}
Using these notations, ResNTK on the sphere \eqref{eq:app ResNTK} can be written as
\begin{equation} \label{eq:NormResNTK}
    \resntkl{L} = \frac{1}{2L} \sum_{\ell=1}^L P_{\ell+1,L}(\kappa_1(u_{\ell-1,L})+u_{\ell-1,L}\kappa_0(u_{\ell-1,L})).
\end{equation}
We next bound the distance of each layer from $\kappa_1(u_0) + u_0 \kappa_0 (u_0)$ from above. In the derivation below we apply several times the mean value theorem, i.e., $\exists$ $c\in[a,b]$, such that $\kappa_1(b) - \kappa_1(a) = \kappa_0(c)(b-a) \leq \kappa_0(b) (b-a)$. This is valid since the derivative of $\kappa_1$ is $\kappa_0$. In addition, $\kappa_0$ is monotonic increasing, so any $c\in [a,b]$ can be replaced by $b$.
\begin{align*}
    & |P_{\ell+1,L}(\kappa_1(u_{\ell-1,L})+u_{\ell-1,L}\kappa_0(u_{\ell-1,L})) - (\kappa_1(u_{0}) + u_{0}\kappa_0(u_{0}))|  \\
    &\le |P_{\ell+1,L}| \cdot |(\kappa_1(u_{\ell-1,L})+u_{\ell-1,L}\kappa_0(u_{\ell-1,L})) - (\kappa_1(u_{0}) + u_{0}\kappa_0(u_{0}))| + |(\kappa_1(u_{0}) + u_{0}\kappa_0(u_{0}))| \cdot |1-P_{\ell+1,L}|  \\
    & \le |\kappa_0(u_{\ell-1,L})(u_{\ell-1,L}-u_{0})| + |\kappa_0(u_{\ell-1,L})u_{\ell-1,L}-\kappa_0(u_{0})u_{0}| + |(\kappa_1(u_{0}) + u_{0}\kappa_0(u_{0}))| \cdot |1-P_{\ell+1,L}|,
\end{align*}
where the last inequality is because $0 < P_{\ell-1,L} \le 1$ and due to the mean value theorem. We next focus on the {\em \underline{first two terms}}
\begin{align*}
    & |\kappa_0(u_{\ell-1,L})(u_{\ell-1,L}-u_{0})| + |\kappa_0(u_{\ell-1,L})u_{\ell-1,L}-\kappa_0(u_{0})u_{0}|  \\
     & \le |\kappa_0(u_{\ell-1,L})(u_{\ell-1,L}-u_{0})| + |\kappa_0(u_{\ell-1,L})u_{\ell-1,L} - \kappa_0(u_{\ell-1,L})u_{0} + \kappa_0(u_{\ell-1,L})u_{0} -\kappa_0(u_{0})u_{0}|  \\
     & \le |\kappa_0(u_{\ell-1,L})(u_{\ell-1,L}-u_{0})| + |\kappa_0(u_{\ell-1,L})u_{\ell-1,L} - \kappa_0(u_{\ell-1,L})u_{0}| + |\kappa_0(u_{\ell-1,L})u_{0} -\kappa_0(u_{0})u_{0}|  \\
     & = 2|\kappa_0(u_{\ell-1,L})(u_{\ell-1,L}-u_{0})|  + |u_{0}(\kappa_0(u_{\ell-1,L}) -\kappa_0(u_{0}))|  \\
     & \le^1 2\kappa_0(u_{\ell-1,L})\frac{\alpha^2}{1+\alpha^2}(\kappa_1(u_{0})-u_{0})(\ell-1) + |u_0| (u_{l-1,L}-u_0) \kappa_0'(c_{l-1,L}) \\
     & =  2\kappa_0(u_{\ell-1,L})\frac{\alpha^2}{1+\alpha^2}(\kappa_1(u_{0})-u_{0})(\ell-1) + |u_0| (u_{l-1,L}-u_0)\frac{1}{\pi \sqrt{1-c_{\ell-1,L}^2}}  \\
     & \le^2 2\kappa_0(u_{\ell-1,L})\frac{\alpha^2}{1+\alpha^2}(\kappa_1(u_{0})-u_{0})(\ell-1) + \frac{|u_{0}|(\kappa_1(u_{0})-u_{0})(\ell-1)}{\pi \sqrt{1-c_{\ell-1,L}^2}}\frac{\alpha^2}{1+\alpha^2}
\end{align*}
where $\le^1$ is obtained by applying \eqref{eq:UpperBound_U} and the mean value theorem for $\kappa_0$ with $c_{l-1,L} \in [u_0, u_{l-1,L}]$, and $\le^2$ too is obtained by applying \eqref{eq:UpperBound_U}.

{\em \underline{Third term}}
\eqref{eq:1mP} and the monotonicity of $\kappa_0$ yield
\begin{align*}
    & |(\kappa_1(u_{0}) + u_{0}\kappa_0(u_{0}))| \cdot |1-P_{\ell+1,L}| \le |(\kappa_1(u_{0}) + u_{0}\kappa_0(u_{0}))| \cdot  \frac{\alpha^2}{1+\alpha^2}(L-\ell-\sum_{i=\ell}^{L-1}\kappa_0(u_{i,L}))  \\
    & \le |(\kappa_1(u_{0}) + u_{0}\kappa_0(u_{0}))| \cdot  \frac{\alpha^2}{1+\alpha^2}(L-\ell)(1-\kappa_0(u_{0}))
\end{align*}
To recap, the upper bound for each layer is
\begin{align}
\label{eq:bound_for_one_layer}
    & |P_{\ell+1,L}(\kappa_1(u_{\ell-1,L})+u_{\ell-1,L}\kappa_0(u_{\ell-1,L})) - (\kappa_1(u_{0}) + u_{0}\kappa_0(u_{0}))|  \\
    & \le 2\kappa_0(u_{\ell-1,L})\frac{\alpha^2}{1+\alpha^2}(\kappa_1(u_{0})-u_{0})(\ell-1) + \frac{|u_{0}|(\kappa_1(u_{0})-u_{0})(\ell-1)}{\pi \sqrt{1-c_{\ell-1,L}^2}}\frac{\alpha^2}{1+\alpha^2}   \nonumber\\
    & + |(\kappa_1(u_{0}) + u_{0}\kappa_0(u_{0}))| \cdot  \frac{\alpha^2}{1+\alpha^2}(L-\ell)(1-\kappa_0(u_{0})). \nonumber
\end{align}
We would like next to derive a bound for the entire kernel, i.e., to bound from above the following expression  
\begin{align*}
    & |\resntkl{L}(u_0)-\ntkl{1}(u_0)| = \left|\frac{1}{2L} \sum_{\ell=1}^L\Bigg\{P_{\ell+1,L}(\kappa_1(u_{\ell-1,L})+u_{\ell-1,L}\kappa_0(u_{\ell-1,L}))\Bigg\} - \frac{1}{2}(\kappa_1(u_{0})+u_{0}\kappa_0(u_{0})) \right|  \\
    & = \left| \frac{1}{2L} \sum_{l=1}^{L} \Bigg\{ P_{\ell+1,L}(\kappa_1(u_{\ell-1,L})+u_{\ell-1,L}\kappa_0(u_{\ell-1,L})) - (\kappa_1(u_{0})+u_{0}\kappa_0(u_{0}))\Bigg\} \right|   \\
    & \le^3  \frac{1}{2L} \frac{\alpha^2}{1+\alpha^2} \sum_{\ell=1}^L \left\{ 2\kappa_0(u_{\ell-1,L})(\kappa_1(u_{0})-u_{0})(\ell-1) + \frac{|u_{0}|(\kappa_1(u_{0})-u_{0})(\ell-1)}{\pi \sqrt{1-c_{\ell-1,L}^2}} 
     + |(\kappa_1(u_{0}) + u_{0}\kappa_0(u_{0}))|  (L-\ell)(1-\kappa_0(u_{0})) \right\}  \\
    & \le^4 \frac{1}{2L} \frac{\alpha^2}{1+\alpha^2} \sum_{\ell=1}^L \left(2(\kappa_1(u_{0})-u_{0})(\ell-1) + \frac{|u_{0}|(\kappa_1(u_{0})-u_{0})(\ell-1)}{\pi \sqrt{1-(1-\delta')^2}}\right) \\
    & +  \frac{1}{2L} \frac{\alpha^2}{1+\alpha^2} |(\kappa_1(u_{0}) + u_{0}\kappa_0(u_{0}))| (1- \kappa_0(u_{0}))\frac{L(L-1)}{2} \\
    & = \frac{L(L-1)}{2} \frac{1}{2L} \frac{\alpha^2}{1+\alpha^2} [2(\kappa_1(u_{0})-u_{0}) + \frac{|u_{0}|(\kappa_1(u_{0})-u_{0})}{\pi \sqrt{1-(1-\delta')^2}} + |(\kappa_1(u_{0}) + u_{0}\kappa_0(u_{0}))| (1-\kappa_0(u_{0}))] \\
    & = \frac{L-1}{4} \frac{\alpha^2}{1+\alpha^2} [2(\kappa_1(u_{0})-u_{0}) + \frac{|u_{0}|(\kappa_1(u_{0})-u_{0})}{\pi \sqrt{1-(1-\delta')^2}} + |(\kappa_1(u_{0}) + u_{0}\kappa_0(u_{0}))| (1-\kappa_0(u_{0}))]
\end{align*}
where $\le^3$ is directly by applying \eqref{eq:bound_for_one_layer}, and $\le^4$ relies on the fact that $0 \le \kappa_0(s)\le 1$ and the following argument.  
We would like to bound from above the term $\frac{1}{\sqrt{1-c_{l-1,L}^2}}$ for $c_{l-1,L} \in [u_0, u_{l-1,L}]$. Since we have
$$-1 +\delta' \leq -1+\delta \leq u_0 \leq \ldots \leq u_{L-1,L} \leq  1 - \delta \leq 1-\delta',$$ 
it follows that 
$\frac{1}{\sqrt{1-c_{l-1,L}^2}} \leq \frac{1}{\sqrt{1-(1-\delta')^2}}$.

Since for $\alpha = L^{-\gamma}$ we have  $\frac{\alpha^2}{1+\alpha^2}= \frac{1}{1+L^{2\gamma}}$ we obtain
\begin{align*}
    & |\resntkl{L}(u_0)-\ntkl{1}(u_0)| 
    \le \\
    & \frac{L-1}{4} \frac{1}{1+L^{2\gamma}} \left[ 2(\kappa_1(u_{0})-u_{0}) + \frac{|u_{0}|(\kappa_1(u_{0})-u_{0})}{\pi \sqrt{1-(1-\delta')^2}} + |(\kappa_1(u_{0}) + u_{0}\kappa_0(u_{0}))| \cdot(1+\kappa_0(u_{0})) \right] \le \\
    & L^{1-2\gamma} \left[2(\kappa_1(u_{0})-u_{0}) + \frac{|u_{0}|(\kappa_1(u_{0})-u_{0})}{\pi \sqrt{1-(1-\delta')^2}} + |(\kappa_1(u_{0}) + u_{0}\kappa_0(u_{0}))| \cdot(1+\kappa_0(u_{0})) \right]
\end{align*}
Hence the bound is $O(L^{1-2\gamma})$, which means that for any $0.5 < \gamma \le 1$, ResNTK converges as  $L \xrightarrow{} \infty$ to FC-NTK for 2-Layer MLP.
\end{proof}

\begin{lemma}\label{eq:supporting_inequality}
    For $a_k \in [0,1]$, it holds that $1 - \prod_{k=1}^{n}(1-a_k) \le \sum_{k=1}^{n} a_k$
\end{lemma}
\begin{proof}
    By induction. The lemma holds trivially for $k=1$. Assume the lemma holds for $k \le n-1$, then 
    \begin{align*}
        & 1 - \prod_{k=1}^{n}(1-a_k) = 1- (1-a_n) \left( \prod_{k=1}^{n-1}(1-a_k) \right) = 1 - \prod_{k=1}^{n-1}(1-a_k) + a_n  \prod_{k=1}^{n-1}(1-a_k) \\ & \leq
         \sum_{k=1}^{n-1} a_k +a_n  \prod_{k=1}^{n-1}(1-a_k) \le \sum_{k=1}^{n} a_k.
    \end{align*}
    
\end{proof}

\end{document}